\documentclass[12pt]{article}

\usepackage{amssymb, amsmath}
\usepackage[mathscr]{eucal}
\usepackage{graphics}
\usepackage{pifont}
\usepackage{xcolor}
\usepackage{algorithm2e}

 \usepackage{amsmath,amsthm, amsfonts,amssymb} 
    \usepackage[mathscr]{eucal}
    \usepackage{verbatim}
    \usepackage{cases}
    \usepackage{graphicx}
    \usepackage{url}
    \usepackage{tensor}
    \usepackage{float}
    \usepackage{url}
    \usepackage{setspace}

    \newtheorem{theorem}{Theorem}[section]
    \newtheorem{lemma}{Lemma}[section]
    
    \newtheorem{remark}{Remark}[section]
    
    \newtheorem{proposition}{Proposition}[section]

    \numberwithin{equation}{section}

    \numberwithin{equation}{section}
    \numberwithin{figure}{section}
    
\newcommand{\blind}{0}

    \newcommand{\wh}{\widehat}

    \newcommand{\wt}{\widetilde}

\newcommand{\pkg}[1]{{\normalfont\fontseries{b}\selectfont #1}}
\let\proglang=\textsf

\DeclareFontFamily{U}{mathx}{\hyphenchar\font45}
\DeclareFontShape{U}{mathx}{m}{n}{<-> mathx10}{}
\DeclareSymbolFont{mathx}{U}{mathx}{m}{n}
\DeclareMathAccent{\wb}{0}{mathx}{"73}

\begin{document}

\def\spacingset#1{\renewcommand{\baselinestretch}%
{#1}\small\normalsize} \spacingset{1}


\if0\blind
{
  \title{\bf Space Partitioning and Regression Mode Seeking via a Mean-Shift-Inspired Algorithm}
  \author{Wanli Qiao\\
    Department of Statistics, George Mason University\\
    and \\
    Amarda Shehu \\
    Department of Computer Science, George Mason University}
  \maketitle
} \fi

\if1\blind
{
  \bigskip
  \bigskip
  \bigskip
  \begin{center}
    {\LARGE\bf Mode Seeking for Regression Functions via a Mean-Shift-Inspired Algorithm}
\end{center}
  \medskip
} \fi

\bigskip
\begin{abstract}
The mean shift (MS) algorithm is a nonparametric method used to cluster sample points and find the local modes of kernel density estimates, using an idea based on iterative gradient ascent. In this paper we develop a mean-shift-inspired algorithm to estimate the modes of regression functions and partition the sample points in the input space. We prove convergence of the sequences generated by the algorithm and derive the non-asymptotic rates of convergence of the estimated local modes for the underlying regression model. We also demonstrate the utility of the algorithm for data-enabled discovery through an application on biomolecular structure data. An extension to subspace constrained mean shift (SCMS) algorithm used to extract ridges of regression functions is briefly discussed. 
\end{abstract}

\noindent%
\vfill

\newpage
\spacingset{1.5} 

\section{Introduction}

The mean-shift (MS) algorithm is a well-known method to cluster sample points and find the local modes of kernel density estimators (KDE) using a gradient ascent idea. This algorithm was introduced by Fukunaga and Hostetler (1975), and was generalized by Cheng (1995). It finds wide applications in image segmentation (see Comaniciu and Meer, 2002) and object tracking (see Comaniciu et al., 2003). The algorithm has thus far no counterpart that partitions sample points in a regression setting and estimates the local modes of regression functions. In this paper we propose a regression mean shift algorithm to fill this gap and study the theoretical properties of our mode estimators for regression functions.  

Let $(X,Y)\in\mathbb{R}^d\times\mathbb{R}$ be a random pair, and $r(x)=\mathbb{E}(Y|X=x)$ be the regression function. Suppose that we observe i.i.d. sample points $(X_1,Y_1),\cdots,(X_n,Y_n)$ that have the same joint distribution as $(X,Y)$. Here the goals are (1) to estimate the set of local modes of $r$, and (2) to partition the input space or the points $\{X_1,\cdots,X_n\}$ according to their connection with the estimated local modes. The plug-in method is a natural way of achieving goal (1). In other words, suppose that we have a good estimator $r_n$ of $r$, we can use the local modes of $r_n$ as the estimators of the local modes of $r$. In fact, this idea has been used in M\"{u}ller (1985, 1989) using the Gasser-M\"{u}ller (GM) kernel regression estimator, and in Ziegler (2002) using the Nadaraya-Watson (NW) regression estimator. However, the plug-in approach does not directly render an algorithm to find the local modes, which is usually challenging because the local modes are only implicitly defined through the regression estimators. As a special case of goal (2), spatial partitioning has many applications in, for example, clustering for house price (Liu et al, 2016), segregated homogeneous neighborhoods studied in sociology (Legewie, 2018), and division of disease risk zones in epidemiology (Gaudart et al. 2005). The regression MS algorithm we propose in this paper uses a modal clustering idea and can be simultaneously useful for these two goals. 

For clustering and mode estimation related to regression models, there is a MS-type algorithm called the conditional mean shift (CMS) algorithm, developed by Einbeck and Tutz (2006). The CMS algorithm is used to estimate the local modes of $f(y|x)$ which is the conditional density function of $Y$ given $X=x$, The algorithm searches for local modes in the space of $y$, with its output indexed by $x$, and has been used in nonparametric modal regression studied by Chen et al. (2016). Note that the CMS algorithm is essentially still an algorithm to estimate the modes of density functions, while the problem we are studying here is to find the local mode estimators for $\mathbb{E}(Y|X=x)$ in the space of $x$. 

We briefly describe the idea behind the original MS algorithm, in order to elucidate the main difference and challenge in extending the MS algorithm to the regression setting. Let $f$ be a differentiable density function on $\mathbb{R}^d$. For a fixed $a>0$, consider a sequence of points, starting from $x_0\in\mathbb{R}^d$, defined iteratively by 
\begin{align}\label{msmodel}
x_\ell = x_{\ell - 1} + a  \frac{\nabla f(x_{\ell - 1})}{f(x_{\ell - 1})},\; \ell\geq 1.
\end{align}
Since $\nabla \log f(x) =\frac{\nabla f(x)}{f(x)}$ for all $x\in\mathbb{R}^d$ such that $f(x)>0,$ the procedure in (\ref{msmodel}) can be understood as a gradient ascent algorithm, with $x_\infty:=\lim_{\ell\rightarrow\infty} x_\ell$, if it exists, as a local mode of $f$ under regularity conditions. With a random sample drawn from $f$, one can get a KDE $\wh f$ defined in (\ref{KDE}) below, replace $f$ by $\wh f$ in the above iterative procedure and generate a sequence $\wh x_j$, $j=0,1,\cdots,$ with $\wh x_0=x_0,$ so that $\wh x_\infty:=\lim_{\ell\rightarrow\infty} \wh x_\ell$ is used as an estimate of $x_\infty$. More specifically, the MS algorithm implicitly uses $a\propto h^2$, where $h$ is the bandwidth of the KDE. The gradient ascent nature of the MS algorithm has been studied in Arias-Castro et al. (2016). 

One appealing feature of the MS algorithm, which is perhaps also why it is so popular, is that, the convergence of the algorithm can be guaranteed under some mild conditions when the kernel function is appropriately chosen. There is no requirement for the step length, i.e. the quantity $a$ in (\ref{msmodel}), which is in fact implicitly determined by $h$ in the MS algorithm. See Ghassabeh (2015), and Yamasaki and Tanaka (2020). When the MS algorithm is applied to modal clustering, the number of clusters does not need to pre-specified, but rather depends on the chosen bandwidth. 

It does not appear straightforward to extend the above idea to a MS-type algorithm for regression models, related to the following aspects:
\begin{itemize}
\item[(1)] The regression function $r$ and its estimator $r_n$ are not always non-negative, so that it is not always meaningful to consider $\log r$ or $\log r_n$ directly, while it seems that the logarithm transformation plays a critical role in the convergence property of the MS algorithm without any requirement for the bandwidth choice; 
\item[(2)] The regression function $r$ has a quotient form as a conditional expectation. The regression estimators that adopt a similar form (such as the NW regression estimator) have more tedious gradients than those of KDE, which makes the mean-shift implementation using such estimators no longer enjoy the same convergence property as the original MS-type algorithm. See Remark~\ref{NWest} below for more discussions. 
\end{itemize}

%
Briefly speaking, we handle the above two issues in the following way: For (1), we apply a positive transformation to the observed response variable $Y_1,\cdots,Y_n$; For (2), with the transformed response variables we use a regression estimator developed in Mack and M\"{u}ller (1989), which is a variant of the NW kernel estimator, but enjoys a simpler form of gradients (and higher order derivatives). With this equipment, here is a summary of our contributions in this paper.

\begin{itemize}
\item[1.] We present a regression mean shift algorithm that is used to partition the sample points in the input sapce and estimate the local modes of $r$. We prove the convergence of the algorithm under mild conditions, which does not have a requirement for the bandwidth (see Theorem~\ref{converres}).
\item[2.] We give non-asymptotic uniform rates of convergence of the Hausdorff distance of the sets of local modes between our estimator and the truth (see Theorems~\ref{modebound1} and ~\ref{modebound2}).
\end{itemize}


We organize the paper as follows. First we present our regression mean shift algorithm in Section~\ref{sec:algorithm} with its convergence proved. Section~\ref{sec:theory} includes theoretical study for the mode estimators. A bandwidth selection strategy for our regression MS algorithm is discussed in Section~\ref{sec:bwselection}. It is followed by simulation and case studies in Section~\ref{sec:numerical}, where in particular we show the application of our algorithm to biomolecular structure datasets. In Section~\ref{sec:discussion} we discuss an extension to a subspace constrained version of our algorithm to extract ridges of regression functions. The proofs are given in Section~\ref{sec:proofs} and the appendix.

\section{Regression Mean Shift Algorithm}\label{sec:algorithm}

Denote the marginal probability density function of $X$ by $f$, and let 
\begin{align}\label{KDE}
\wh f(x) = \frac{1}{nh^d} \sum_{i=1}^n K\Big(\frac{x-X_i}{h}\Big),\; x\in\mathbb{R}^d
\end{align}
be the KDE of $f$, where $K$ is a kernel density function on $\mathbb{R}^d$, and $h>0$ is a bandwidth. Denote $K_h(x)=K(x/h)$. We have the following kernel regression estimator of $r(x)$, proposed by Mack and M\"{u}ller (1989):
\begin{align}
\wh r(x) =  & \frac{1}{nh^d} \sum_{i=1}^n \frac{Y_iK_h(x-X_i) }{\wh f(X_i)} .
\end{align} 
Note that if we take derivatives of $\wh r$, the differential operator only needs to be applied to the numerator, which helps avoid the tedious form of the derivatives of, say, the NW regression estimator.

Let $K$ be a spherically symmetric kernel with profile $k:\mathbb{R}_{\geq 0}\rightarrow\mathbb{R}$, that is, $K(x)=c_{k,d}k(\|x\|^2)$, where $c_{k,d}>0$ is a normalization factor such that $c_{k,d}^{-1}=\int_{\mathbb{R}^d}k(\|x\|^2)dx$. Examples of $K$ include the Gaussian kernel and Epanechnikov kernel. Then we can write 
\begin{align}\label{MMest}
\wh r(x) = \frac{c_{k,d}}{nh^d}\sum_{i=1}^n \frac{Y_ik(\|x-X_i\|^2/h^2)}{\wh f(X_i)}.
\end{align}

We will transform $Y_i,i=1,\cdots,n$ by applying a strictly increasing positive function $\xi:\mathbb{R}\rightarrow\mathbb{R}_{>0}$. We consider the following two choices of $\xi$.
\begin{itemize}
\item[{\bf T1}]{(Transformation 1):} $\xi$ is a deterministic bounded function. For example, $\xi(x) = \text{logistic(x)}$.
\item[{\bf T2}]{(Transformation 2):}  $\xi$ is a random function depending on $Y_{[n]}: = \min_i Y_i$ such that $\xi(x) = x + \pi(Y_{[n]})$, where $\pi(Y_{[n]})= (-Y_{[n]}+c_0)\mathbf{1}(Y_{[n]} <c_0)$ for some positive constant $c_0.$ Note that $\min_i \xi(Y_i)\geq c_0.$
\end{itemize}
 Let $\wt Y=\xi(Y)$ and $\wt Y_i= \xi(Y_i)>0$, $i=1,\cdots,n$. Define
 \begin{align}\label{rstarx}
\wh r_*(x) =  \wh r_{*,k}(x) = \frac{c_{k,d}}{nh^d} \sum_{i=1}^n \frac{\wt Y_ik(\|x-X_i\|^2/h^2) }{\wh f(X_i)} ,
\end{align} 
which is considered as an estimator of $\wt r(x):= \mathbb{E}(\wt Y|X=x)$. Define $g(x)=-k^\prime(x)$ for all $x\in[0,\infty)$, assuming that the derivative exists. For any $x\in\mathbb{R}^d$, denote 
\begin{align}\label{wixdef}
w_i(x)=\frac{g\left( \|x-X_i\|^2/h^2 \right)}{\wh f(X_i)},
\end{align} 
and define
\begin{align}\label{meanshiftform}
\wh m_*(x) = \frac{\sum_{i=1}^n w_i(x)\wt Y_iX_i}{\sum_{i=1}^n w_i(x)\wt Y_i} - x,
\end{align}
which is called the \emph{regression mean shift}. Note that we have the following relation.
\begin{align}\label{grdexp}
\nabla \wh r_{*}(x) = \nabla \wh r_{*,k}(x) = & \frac{2c_{k,d}}{nh^{d+2}}\sum_{i=1}^n \frac{-\wt Y_i(x-X_i)g( \|x-X_i\|^2/h^2)}{\wh f(X_i)} \nonumber\\
 = & \frac{2 c_{k,d}}{h^2 c_{g,d}}\wh r_{*,g}(x) \wh m_*(x).
\end{align}
Here $\wh r_{*,g}$ is defined as in (\ref{rstarx}) for $\wh r_{*,k}$, where we replace $k$ by $g$. In other words, the regression mean shift $\wh m_*(x)$ is proportional to $\frac{\nabla \wh r_{*,k}(x)}{\wh r_{*,g}(x)}$ up to a constant coefficient, so that the following regression MS algorithm can also be understood as a gradient ascent algorithm.

\paragraph{Regression MS algorithm.} Our regression MS algorithm is as follows. Let $z_0$ be a starting point on the domain of $r$ (e.g., one of $X_i$'s). Obtain $z_1,z_2,\cdots,$ iteratively from 
\begin{align}\label{gms}
z_{j+1} = \wh m_*(z_j) + z_j, \; j=1,2,\cdots.
\end{align}
The limit of the sequence $\{z_0,z_1,\cdots\}$ is considered as an estimator of a local mode of $r$. In practice, the algorithm stops when the distance between two consecutive points $\|z_{j+1} - z_{j}\|$ is less than a pre-specified small threshold. 

The following lemma is analogous to Theorem 1 in Comaniciu and Meer (2002). 
\begin{lemma}\label{gradlemma}
If $k$ is convex and strictly decreasing such that $-\infty<k^\prime(x)<0$ for all $x\geq 0$, then we have 
\begin{itemize}
\item[(1)] $\wh r_*(z_{j})$ converges, 
\item[(2)] $\|z_{j+1} - z_{j}\|\rightarrow 0$, and 
\item[(3)] $\nabla\wh r_*(z_j)\rightarrow0$, as $j\rightarrow\infty$.
\end{itemize}
\end{lemma}

\begin{remark}\label{NWest}$\;$\\
\emph{a). A related and alternative method is to use the discretized gradient ascent with a constant step length based on a smooth estimator of the regression function (such as the NW regression estimator). However, this method requires the step length to be chosen sufficiently small; otherwise it is well-known that there is an overshooting problem and the sequence can diverge (see Chapter 1, Bertsekas, 1999). In contrast, the convergence of our regression MS algorithm does not rely on requirements for step length.\\
b). It is natural to wonder if a MS-type algorithm can be developed based on the gradient of the NW regression estimator, in a way similar to (\ref{grdexp}). The analysis below shows such an algorithm is not effective in general. The NW regression estimator using $\{(X_i,\wt Y_i),i=1,\cdots,n\}$ is given by
\begin{align}
\wh r_{\text{NW}}(x) = \frac{\sum_{i=1}^n \wt Y_ik(\|x-X_i\|^2/h^2)}{\sum_{i=1}^n k(\|x-X_i\|^2/h^2)}.
\end{align}
Let $w_i^k(x) = k(\|x-X_i\|^2/h^2)$ and $w_i^g(x) = g(\|x-X_i\|^2/h^2)$. Let $$w_i^{*}(x) = \wt Y_iw_i^g(x) \Big[\sum_{i=1}^n w_i^k(x)\Big] - w_i^g(x) \Big[\sum_{i=1}^n \wt Y_iw_i^k(x)\Big].$$ It follows from a straightforward calculation that
\begin{align}
\nabla \wh r_{\text{NW}}(x) 
=& \frac{2}{h^2}\frac{\sum_{i=1}^n w_i^{*}(x)X_i - \sum_{i=1}^n w_i^{*}(x)x}{[\sum_{i=1}^n w_i^k(x)]^2}.
\end{align}
If $g\propto k$ (which happens, for example, when $K$ is the Gaussian kernel), then $\sum_{i=1}^n w_i^{*}(x)\equiv 0$, and this makes it hopeless to get a mean shift form as given in (\ref{meanshiftform}). Now suppose a kernel can be chosen such that $\sum_{i=1}^n w_i^{*}(x)\neq 0$. Then we can write  
 \begin{align}
\nabla \wh r_{\text{NW}}(x) = \frac{2}{h^2} \wh s(x) \wh m_{\text{NW}}(x).
\end{align}
where $\wh m_{\text{NW}}(x) = \frac{\sum_{i=1}^n w_i^*(x) X_i}{\sum_{i=1}^n w_i^*(x)} - x$, which corresponds to our regression mean shift, and $\wh s(x) = \frac{\sum_{i=1}^n w_i^*(x)}{[\sum_{i=1}^n w_i^k(x)]^2}$. Note that in general the sign of $w_i^*(x)$ is not always positive, and hence it is not clear if a similar result as given in Lemma~\ref{gradlemma} holds for $\wh r_{\text{NW}}(x).$ In fact, a simulation we ran shows the converge of the sequence generated by $\wh m_{\text{NW}}$ using the Epanechnikov kernel for $g$ is problematic. It appears that the mean shift idea and the quotient form of the NW regression estimator are not compatible. }
\end{remark}
%
%
%
We need to additionally assume that the critical points of $\wh r_*(x)$ are isolated, in order to have the convergence of our regression mean shift algorithm. The proof of the following theorem is similar to that of Theorem 1 in Ghassabeh (2015) for the MS algorithm. 

\begin{theorem}\label{converres}
Suppose that the assumptions in Lemma~\ref{gradlemma} hold. If the critical points of $\wh r_*(x)$ are isolated, then the sequence of $z_j$ converges to one of the critical points of $\wh r_*(x)$ as $j\rightarrow\infty$.
\end{theorem}


\subsection{Basins of Attraction for Regression Functions}\label{sec:partition}

The mode seeking algorithm in (\ref{gms}) can be used to partition the input space into basins of attraction. This can be understood using the framework of Morse theory (See Milnor, 1963). A similar perspective has been used to interpret modal clustering using the MS algorithm. See Ch\'{a}con (2015). Suppose that $\mathcal{X}$ is a compact set of positive volume contained in the support of the density of $X$.  Also suppose that $r$ is a twice differentiable Morse function, meaning that all of its critical points are non-degenerate, that is, the Hessian at each critical point is nonsingular. Let $\mathcal{M}$ be the collection of all local modes of $r$, denoted by $x_1,\cdots,x_m$, where $m$ is the cardinality of $\mathcal{M}$. For any $x\in\mathcal{X}$, let $\phi_x: \mathbb{R} \rightarrow \mathcal{X}$ be the integral curve driven by the gradient of $r$, starting from $x$:
\begin{align*}
\frac{d\phi_x(t)}{dt} = \nabla r(\phi_x(t)), t\in\mathbb{R};\; \phi_x(0) = x.
\end{align*}
Then by the Morse theory, for any $x\in\mathcal{X}$, $\phi_x(\infty):=\lim_{t\rightarrow\infty}\phi_x(t)$ is one of the critical points of $r$. In particular, for $j=1,\cdots,m$, the basins of attraction associated with $x_j$ is
\begin{align*}
C(x_j) : = \{x\in\mathcal{X}: \; \phi_x(\infty)=x_j\},
\end{align*}
which is also called a stable manifold, or ascending manifold in Morse theory. The sets in $\mathcal{C}:= \{C(x_j), j=1,\cdots,m\}$ are disjoint, and their union covers $\mathcal{X}$ except for a set of zero Lebesgue measure. 

Let us first consider the deterministic transformation $\xi$ in {\bf T1}. Under regularity conditions, one can show that $r$ and $\wt r$ have the same ascending manifolds (see Lemma~\ref{morse} below). The regression estimator $\wh r_*$ is used to estimate $\wt r$, and the sequence (\ref{gms}) is viewed as discretized estimation of trajectories of the integral curves driven by $\nabla \log \wt r$. Let $\wh{\mathcal{M}}$ be the set of all local modes of $\wh r_*$, consisting of $\wh x_1,\cdots, \wh x_{\wh m}$, where $\wh m$ is its cardinality. For any $x\in\mathcal{X}$, let $\wh\phi_x(\infty)$ be the limit of the sequence in (\ref{gms}) when $z_0=x.$ Define 
\begin{align*}
\wh C(\wh x_j) : = \{x\in\mathcal{X}: \; \wh\phi_x(\infty)=\wh x_j\}.
\end{align*}
Then $\wh {\mathcal{C}}:=\{\wh C(\wh x_j): \; j=1,\cdots,\wh m\}$ also gives a partition of $\mathcal{X}$ (up to a small set not covered), and can be used to estimate $\mathcal{C}$. 

For the transformation $\xi$ given in {\bf T2}, the idea is similar. Using this transformation, the regression estimator $\wh r_*$ is used to estimate 
\begin{align}\label{barr}
\wb r := r + \pi(Y_{[n]}).
\end{align}
Notice that $\wb r$ and $r$ has the same ascending manifolds, assuming that $\pi(Y_{[n]})$ is bounded. So again $\wh {\mathcal{C}}$ gives an approximate partition of $\mathcal{X}$ and can be used to estimate $\mathcal{C}$. 

For both transformations, the sample points $X_1,\cdots,X_n$ in the input space can be partitioned based on which basins of attraction they belong to, and this idea is used in the simulation and case studies in Section~\ref{sec:numerical}.

\section{Theoretical Analysis of the Mode Estimators}\label{sec:theory}

In this section we study the theoretical properties of $\wh r_*$ and its modes as direct plug-in estimators of the modes of $r$. We derive the uniform rate of convergence of $\wh r_*$, which further gives the rate of convergence of its local modes in Hausdorff distance. Our results are non-asymptotic, and the derived rates of convergence for the local mode estimation match the minimax rate of mode estimation for density functions up to a logarithm factor (see Remark~\ref{optimalrates}). Note that the results in Mack and M\"uller (1989) can only be used to provide the uniform consistency, but not the uniform rates of convergence of their estimator.

We will use the following notation. For any $d$-tuple $\alpha=(\alpha_1,\cdots,\alpha_d)\in\mathbb{N}^d$, let $|\alpha| = \alpha_1+\cdots+\alpha_d$. For an $|\alpha|$ times differentiable function $g:\mathbb{R}^d\rightarrow \mathbb{R}$, denote $\partial^{\alpha} g(x) =  \frac{\partial^{|\alpha|} }{\partial^{\alpha_1}x_1\cdots \partial^{\alpha_d}x_d} g(x),\; x\in\mathbb{R}^d.$ For a composition of functions $g_1\circ g_2$, we write $\partial_x^{\alpha} g_1(g_2(x))=\partial^\alpha(g_1\circ g_2)(x)$. Let $\nabla g$ and $\nabla^2 g$ be the gradient and Hessian of $g$, respectively. For any real numbers $a,b$, let $a\wedge b=\min(a,b)$ and $a\vee b=\max(a,b)$. For simplicity of notation, for any $n\geq 1$, $h>0$, $j\in\mathbb{Z}_+$, we denote $\gamma_{n,h}^{(j)} = (nh^{d+2j})^{-1/2}.$

Throughout the paper $\mathcal{X}$ denotes a compact subset of $\mathbb{R}^d$ with strictly positive volume. For any $\delta>0$, let $\mathcal{X}^\delta=\{x\in\mathbb{R}^d: \inf_{t\in\mathcal{X}}\|x-t\|\leq\delta\}$. We will use the following assumptions in our theoretical analysis. 

Assumption {\bf A1}: The marginal density of $X$, denoted by $f$, satisfies $\inf_{x\in\mathcal{X}} f(x)\geq \varepsilon_0$ for a constant $\varepsilon_0>0$. 
%

Assumption {\bf A2}: $f$ has three times continuous bounded derivatives on $\mathcal{X}^\delta$ for some $\delta>0$.

Assumption {\bf A3}: $r$ has three times continuous bounded derivatives on $\mathcal{X}^\delta$  for some $\delta>0$.

Assumption {\bf K}: The kernel $K$ is a spherically symmetric density function with its support contained in the unit ball of $\mathbb{R}^d$. $K$ has three times continuous bounded derivatives on $\mathbb{R}^d$. 
 
\subsection{Transformation 1}

We first consider the transformation $\xi$ in {\bf T1}.  We can write the regression model as follows: 
\begin{align}\label{regressionmodel}
Y_i = r(X_i) + \epsilon_i, \; i=1,\cdots,n,
\end{align}
where $\epsilon_i, i=1,\cdots,n,$ are i.i.d random errors with mean zero. We make the following assumptions.

Assumption {\bf E}: For $i=1,\cdots,n$, each $\epsilon_i$ is independent of $X_i$.

Assumption {\bf T}: $\xi$ is a strictly increasing function on $\mathbb{R}$ with three times continuous bounded derivatives.  Assume that there exist constants $0<C_\ell<C_u<\infty$ such that $\xi(\mathbb{R})\subset[C_\ell, C_u]$.

For a twice differential function $g$, the index of a critical point $x_{\text{crit}}$ of $g$ is the number of negative eigenvalues of $\nabla^2g$ at $x_{\text{crit}}$. We first show that the critical points (including the local modes) of $\wt r(x)$ and $r(x)$ are the same under the above conditions. 

\begin{lemma}\label{morse}
Assume that $r$ is twice differentiable. For $\xi$ in {\bf T1}, under the assumptions {\bf E} and {\bf T}, the critical points of $\wt r$ and $r$ are the same with the same indices. If $r$ is a Morse function, then (1) $\wt r$ is also a Morse function, and (2) the ascending manifolds of $r$ and $\wt r$ are the same. 
\end{lemma}

\begin{remark}
\emph{The above lemma implies that the local modes of $r$ can be estimated by the local modes of $\wh r_*$ as a plug-in approach, because $\wh r_*$ is considered as an estimator of $\wt r$ as shown in Theorems~\ref{derbound} below. }
\end{remark}

In the following theorem we give a non-asymptotic result for the uniform rate of convergence for the difference between $\partial^\alpha \wt r$ and $\partial^\alpha \wh r_*$ for all $|\alpha|\leq 2$.

\begin{theorem}\label{derbound}
For $\xi$ in {\bf T1}, under assumptions {\bf A1}-{\bf A3}, {\bf K}, {\bf E}, and {\bf T}, there exist constants $C>0$, $c>0$ and $h_0>0$ such that for all $|\alpha|\leq 2$, $n\geq 1$, $0<h\leq h_0$, $\tau>1$ satisfying $nh^d\geq c(\tau \vee |\log h|)$ we have 
\begin{align}
&\mathbb{P}^n\Big(\sup_{x\in\mathcal{X}}|\partial^\alpha\wh r_*(x) - \partial^\alpha \wt r(x)| \leq C  \sqrt{\tau \vee |\log h|} \gamma_{n,h}^{(|\alpha|)} + C h^{(3-|\alpha|)\wedge 2}\Big) \geq 1 - 3e^{-\tau}.
\end{align}
\end{theorem}

Let $\lambda_1(x)$ be the largest eigenvalue of $\nabla^2 r(x), x\in\mathcal{X}.$ We can write the set of local modes of $r$ as $\mathcal{M}=\{x\in\mathcal{X}: \nabla r(x) = 0,\; \lambda_1(x)<0\}$, which is assumed to be nonempty. Let $\wh{\mathcal{M}}$ be the set of local modes of $\wh r_*.$ For any two subset $A,B\subset\mathbb{R}^d$, their Hausdorff distance is defined as
\begin{align*}
d_H(A,B) = \max\Big\{\sup_{a\in A}\inf_{b\in B} \|a-b\|, \sup_{b\in B}\inf_{a\in A}\|a-b\|\Big\}.
\end{align*}
To study $d_H(\mathcal{M}, \wh{\mathcal{M}})$, we will use the following perturbation result for the set of local modes.

\begin{lemma}\label{modelemma}
Let $\mathcal{R}$ be a compact subset of $\mathbb{R}^d$ with positive volume, and $\partial\mathcal{R}$ be its boundary. Let $p: \mathcal{U}\rightarrow\mathbb{R}$ be a three times continuously differentiable Morse function, where $\mathcal{U}\supset\mathcal{R}$ is an open subset of $\mathbb{R}^d$. Let $\lambda_1(x)$ be the largest eigenvalue of $\nabla^2 p(x)$, $x\in\mathcal{R}$. Let $\mathcal{M}$ be a set of local modes of $p$ on $\mathcal{R}$, that is, $\mathcal{M} = \{x\in \mathcal{R}: \nabla p(x) = 0, \lambda_1(x)<0)\}$, and $\mathcal{C}= \{x\in \mathcal{R}: \nabla p(x) = 0 \}$ be the set of all critical points of $p$ on $\mathcal{R}$. Assume that $\eta:=\inf_{x\in \mathcal{C}}d(x,\partial \mathcal{R})>0$. Let $\wt p:\mathcal{U} \rightarrow\mathbb{R}$ be a twice differentiable function, and $\wt{\mathcal{M}}$ be the set of local modes of $\wt p$ on $\mathcal{R}$. There exists a constant $c_0>0$ such that if $\sup_{x\in\mathcal{R}} \max_{|\alpha|\leq 2}|\partial^\alpha p(x) - \partial^\alpha\wt p(x)|<c_0$, then $\wt p$ has the same number of local modes as $p$ on $\mathcal{R}$, and $$d_H(\mathcal{M} ,\wt{\mathcal{M}}) \leq \frac{4}{\lambda_*}\max_{x\in\mathcal{M}}\|\nabla \wt p(x) - \nabla p(x)\|,$$
where $\lambda_*:=-\inf_{x\in\mathcal{M}}\lambda_1(x)>0$.
\end{lemma}

\begin{remark}
\emph{The Hausdorff distance between the sets of modes of the true and estimated functions is also studied in Theorem 1 of Chen et al. (2016). As a comparison, our result is given under weaker conditions. In particular, we do not require their assumption (M2), which assumes that there exist $\eta_1>0$ and $C_3>0$ such that $\{x:\|\nabla p(x)\|\leq\eta_1, 0>-\lambda_*/2\geq \lambda_1(x)\}\subset \mathcal{M}^{\lambda_*/(2dC_3)}$.}
\end{remark}

The following theorem gives a non-asymptotic bound for $d_H(\mathcal{M}, \wh{\mathcal{M}})$, as a direct consequence of Lemma~\ref{morse}, Theorem~\ref{derbound}, and Lemma~\ref{modelemma}.

\begin{theorem}\label{modebound1}
For $\xi$ in {\bf T1}, under assumptions {\bf A1}-{\bf A3}, {\bf K}, {\bf E}, and {\bf T}, if $r$ is a Morse function and there are no critical points on the boundary of $\mathcal{X}$, then there exist constants $C>0$, $c>0$, and $h_0>0$ such that for all $n\geq 1$, $0<h\leq h_0$, $\tau>1$ satisfying $nh^{d+4}\geq c(\tau \vee |\log h|)$ we have with probability at least $1 - 2(d+2)^2e^{-\tau}$ that, $\mathcal{M}$ and $\wh{\mathcal{M}}$ have the same cardinality, and $d_H(\mathcal{M},\wh{\mathcal{M}}) \leq C(\sqrt{\tau \vee |\log h|} \gamma_{n,h}^{(1)} +h^2).$
\end{theorem}

\begin{remark}\label{optimalrates}$\;$\\
\emph{a). When $h=O(n^{-\frac{1}{d+6}})$, it is straightforward to show that $d_H(\mathcal{M},\wh{\mathcal{M}}) = O( n^{-\frac{2}{d+6}}\sqrt{\log n} )$ almost surely, by applying the Borel-Cantelli Lemma to the above result with $\tau=2\log n.$ This matches the minimax rate of convergence up to $\sqrt{\log n}$ of mode estimation for density functions, as given in Theorem 3 of Tsybakov (1990) with the smoothness parameter $\beta=3$ therein. The minimax rate of mode estimation for regression functions under a similar smoothness assumption is unknown to our best knowledge, and it is expected to be the same as that for density functions. As a side note, in the case of a unique mode, the mode estimator using the k-NN regression, which is studied in Jiang (2019), matches the minimax rate in Tsybakov (1990) with $\beta=2$, when $k$ is appropriately chosen.\\
b). It can be seen from the proof that the constants in this theorem in fact do not depend on the magnitude (e.g. variance) of noise $\epsilon_i$. This is unlike the case for the estimation of regression function itself, because we utilize a bounded transformation $\xi$ and the property in Lemma~\ref{morse}.}
\end{remark}

%
%

\subsection{Transformation 2}

Next we consider the transformation $\xi$ in {\bf T2}.  We will replace assumption {\bf E} by the following assumption in our analysis.

Assumption {\bf E$^\prime$}: There exists a constant $B\in(0,\infty)$ such that $|Y|\leq B$ almost surely.

Under assumption {\bf E$^\prime$}, we have $0\leq \pi(Y_{[n]}) \leq B + c_0$ almost surely. We can write $\wh r_*(x) = \wh r(x) + \pi(Y_{[n]}) \wh t(x)$, where 
\begin{align}\label{wht}
\wh t(x) = \frac{1}{nh^d} \sum_{i=1}^n \frac{K_h(x-X_i)}{\wh f(X_i)},
\end{align}
which is an estimator of unity. We consider $\wh r_*$ as an estimator of $\wb r$, which is given in (\ref{barr}). We still denote  the set of local modes of $\wh r_*$ by $\wh{\mathcal{M}}$, which can be used to estimate $\mathcal{M}$, because the set of modes of $\wb r$ is the same as that of $r$. The following result is similar to Theorem~\ref{modebound1}.
\begin{theorem}\label{modebound2}
For $\xi$ in {\bf T2}, under assumptions {\bf A1}-{\bf A3}, {\bf K}, and {\bf E$^{\prime}$}, if $r$ is a Morse function and there are no critical points on the boundary of $\mathcal{X}$, then there exist constants $C>0$, $c>0$, and $h_0>0$ such that for all $n\geq 1$, $0<h\leq h_0$, $\tau>1$ satisfying $nh^d\geq c(\tau \vee |\log h|)$ we have with probability at least $1 - 3(d+2)^2e^{-\tau}$ that, $\mathcal{M}$ and $\wh{\mathcal{M}}$ have the same cardinality, and $d_H(\mathcal{M},\wh{\mathcal{M}}) \leq C(\sqrt{\tau \vee |\log h|} \gamma_{n,h}^{(1)} +h^2).$
\end{theorem}
%

\section{Bandwidth Selection}\label{sec:bwselection}

The selection of bandwidth is critical for the finite-sample performance of kernel type estimators. In particular, the bandwidth $h$ determines the number of modes and clusters using the regression mean shift. 
Since our mean shift algorithm can be understood as tracking the discretized gradient integral curves of the estimated regression function, a bandwidth producing good estimators of the gradient of the regression function is expected to be suitable for our regression MS algorithm. This can also be seen from Lemma~\ref{modelemma}. Based on this observation, below we give a bandwidth selection strategy for our regression MS algorithm using a cross validation idea. 

Let $\nabla \wh r_\dagger(x)$ be a nonparametric kernel estimator of the gradient $\nabla \wt r(x)$. For example, $\nabla \wh r_\dagger(x)$ can be the gradient of the NW regression estimator $\wh r_{\text{NW}}(x)$, or the gradient component using the local linear (LL) regression estimator. For $j=1,\cdots,n$, let $\nabla \wh r_{*,(-j)}(x)$ be the gradient estimator as given in (\ref{grdexp}), but using the sample points excluding $(X_j,\wt Y_j)$. The leave-one-out cross-validation error is defined as
\begin{align}
\text{CV}(h) = \frac{1}{n} \sum_{j=1}^n \| \nabla \wh r_\dagger(X_j) - \nabla \wh r_{*,(-j)}(X_j) \|^2.
\end{align}
The least square cross validation bandwidth $h_{\text{LSCV}}$ which minimizes $\text{CV}(h)$ is proposed to be used for our regression mean shift algorithm. Note that $\nabla \wh r_\dagger$ itself requires a bandwidth choice. For LL estimator, one can use the gradient-based method as given in Henderson et al. (2015). For NW gradient estimator, one can scale the optimal bandwidth for NW regression estimator by multiplying a factor $n^{\frac{1}{(d+4)(d+6)}}$. See Chapter 5.5 of Handerson and Parmeter (2015). We adopt the second method in our simulation study. 

\section{Simulations and Applications}\label{sec:numerical}

\subsection{Simulation studies}
We ran simulations to show the effectiveness of our regression mean shift algorithm in partitioning the sample points in the input space and identifying the local modes of a regression function. We considered the model $Y_i=r(X_i)+\epsilon_i$ as in (\ref{regressionmodel}), where $r$ is a bivariate bimodal function. Specifically, $r(x)=f_1(x)+f_2(x)$, where $f_1$ and $f_2$ are the density functions of $\mathscr{N}(\mu_1,\Sigma_1)$ and $\mathscr{N}(\mu_2,\Sigma_2)$, respectively, with $\mu_1=(1,1)^T$, $\Sigma_1=\text{diag}(0.5,0.5)$, $\mu_2=(-1,-1)^T$, $\Sigma_2=\text{diag}(0.3,0.9)$; for $i=1,\cdots,n,$ $\epsilon_i\overset{\text{i.i.d.}}{\sim} \mathscr{N}(0,0.01)$, and $X_i$ is i.i.d. truncated bivariate normal such that $X_1 {\sim} \mathscr{N}(\mu_3,\Sigma_3)$, with $\mu_3=(0,0)^T$ and $\Sigma_3=\text{diag}(1.5,1.5)$ conditional on $X_1\in[-2,2]^2$. In each run $n=200$ data points were generated from the above model as the input of our algorithm and we repeated the procedure for 200 times. We used the Epanechnikov kernel for $g$, and $\xi(x)=1/(1+\exp(-10x))+0.01$ for transformation {\bf T1} and $c_0=0.1$ for transformation {\bf T2}. Figure~\ref{fig:boxplots} shows the boxplots of the number of modes detected by the algorithm using grid points of bandwidth values, overlapped with the relative frequency of correct bimodal identification (red curves). Using the bandwidth selection strategy in Section~\ref{sec:bwselection}, among the 200 replications the relative frequencies that algorithm can correctly find two modes are 78\% for {\bf T1} and 81\% for {\bf T2}, respectively, which are comparable to the peak values in Figure~\ref{fig:boxplots}. When we increased the sample size $n$ to 500, the relative frequencies of correct number of modes reach 91\% and 93.5\%, respectively. These numbers are as high as 94.5\% and 97.5\% when $n=1000$. 

\begin{figure}[h]
\begin{center}
\includegraphics[width=16cm]{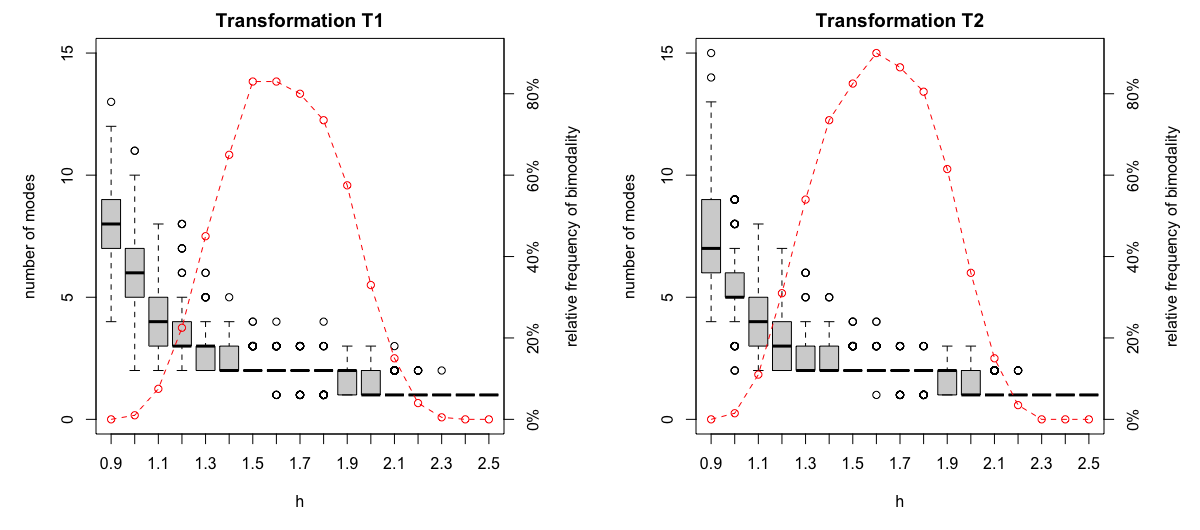}
\end{center}
\caption{Boxplots of the numbers of modes/clusters across different bandwidths for transformations T1 (left), and T2 (right) when $n=200$ for 200 samples, with the overlapping red curves representing the relative frequency of detecting two modes.}
\label{fig:boxplots}
\end{figure}

In Figure~\ref{fig:simulation} we visualize the outcome of the algorithm using a representative random sample of size 200 based on {\bf T1}, which shows the paths of the estimation sequence in our regression MS algorithm, as well as the impact of different bandwidths on the mode estimation results. Not surprisingly, when the bandwidth is small, there tend to be more local modes (or basins), which can also be seen from the boxplots in Figure~\ref{fig:boxplots}. It can be seen that the bandwidth selection strategy given in Section~\ref{sec:bwselection} works well on this sample.

\begin{figure}[h]
\begin{center}
\includegraphics[width=12cm]{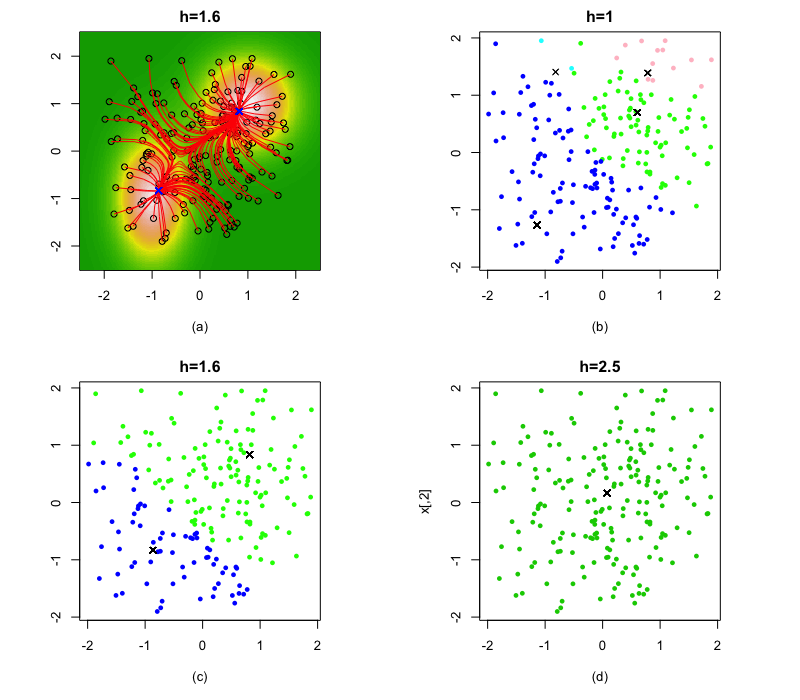}
\end{center}
\caption{Panel (a): The regression function is represented by the color in the background. Black dots are the sample points. Red curves are lines connecting sequential points generated from the algorithm in (\ref{gms}) starting from the sample points. The two $\times$ symbols are the points of convergence. Using the method in Section~\ref{sec:bwselection}, the selected bandwidth is 1.6. Panels (b) - (d) shows the effect of the bandwidth choice. The sample points are partitioned (shown by different colors) according to their points of convergence (represented by $\times$). There are 4 basins when $h=1$, 2 basins when $h=1.6$, and 1 basin when $h=2.5$. }
\label{fig:simulation}
\end{figure}

\subsection{Examples of applications}

\subsubsection{Partitions of protein energy landscapes}

The proposed algorithm can be useful to obtain deep insight about the structure-function relationship in biological molecules (biomolecules). In the application highlighted here, we focus on protein molecules, which are ubiquitous in the cell, and where the three-dimensional structures accessed at equilibrium (under physiological conditions) often regulate a rich set of activities. Figure~\ref{fig:BioApplication} relates the results obtained when the proposed algorithm is employed to organize the three-dimensional structures of the human H-Ras protein by their potential energies. 

The structures (data sets) of the human H-Ras protein are obtained via the biophysical methodology described in Maximova et al. (2017) and Maximova et al. (2018). This work obtains structures for different versions of human H-Ras, the naturally-occurring, also referred to as the wildtype (WT) version, and mutated versions, known as variants. In the WT version, the protein accesses groups of structures that regulate its activity between an ``on'' and ``off'' state; in the on state, H-Ras instigates cellular reactions that signal the cell to grow; in the off state, such signals stop. In mutated variants, which are found in many disorders, the regulation is disrupted in some manner, but only a view of the space of structures accessible can reveal exactly what, at the structure level, is responsible for dysfunction. The proposed algorithm promises to reveal exactly such organization of structures, as we show here.

From the structures/datasets produced by work in Maximova et al. (2017) and Maximova et al. (2018) for the WT and a common oncogenic variant, Q61L (the naming indicates the position where the naturally-occurring amino acid, ``Q'', has been replaced with a different amino acid, ``L'' in this case), we randomly selected $2000$ structures for each, WT and Q61L\footnote{The data sets are downloadable at \url{https://dx.doi.org/10.21227/331n-7019}}. Each structure (data point) comes with an associated energy value, which sums the physical interactions among the atoms in a particular structure. These energy values (in the original data sets) are all negative, and we used their absolute values in the analysis, so that our regression mean shift algorithm can cluster data points based on the local minima (antimodes) that they converge to.

Each panel in Figure~\ref{fig:BioApplication} organizes the samples by their energies. Each dot corresponds to a three-dimensional structure. The red-to-blue color-coding scheme indicates high-to-low energies. The left panel shows the WT form/variant of the human H-Ras protein; the right panel shows the oncogenic variant known as Q61L.

\begin{figure}[h]
\centering
\begin{tabular}{cc}
\includegraphics[width=0.45\columnwidth]{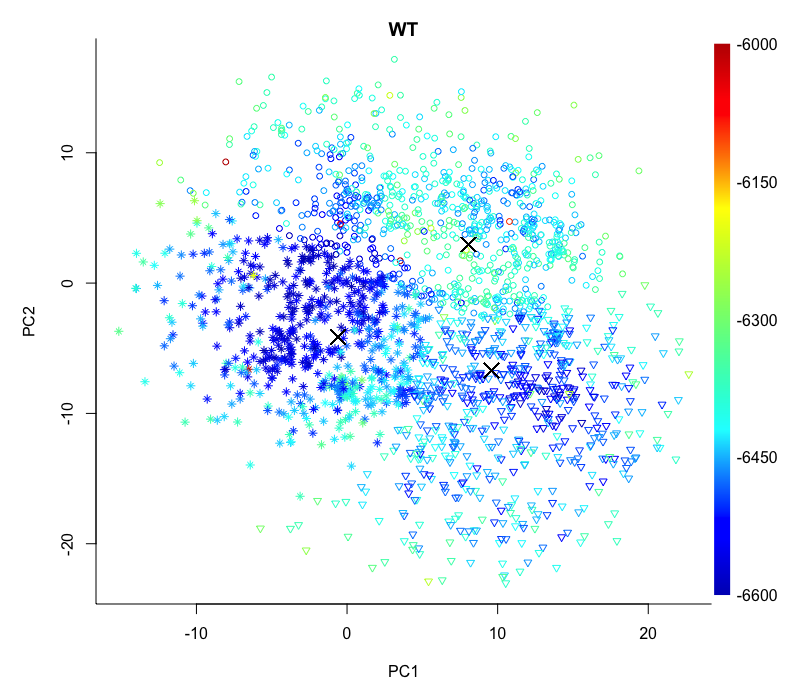} & \hspace*{-3mm}
\includegraphics[width=0.45\columnwidth]{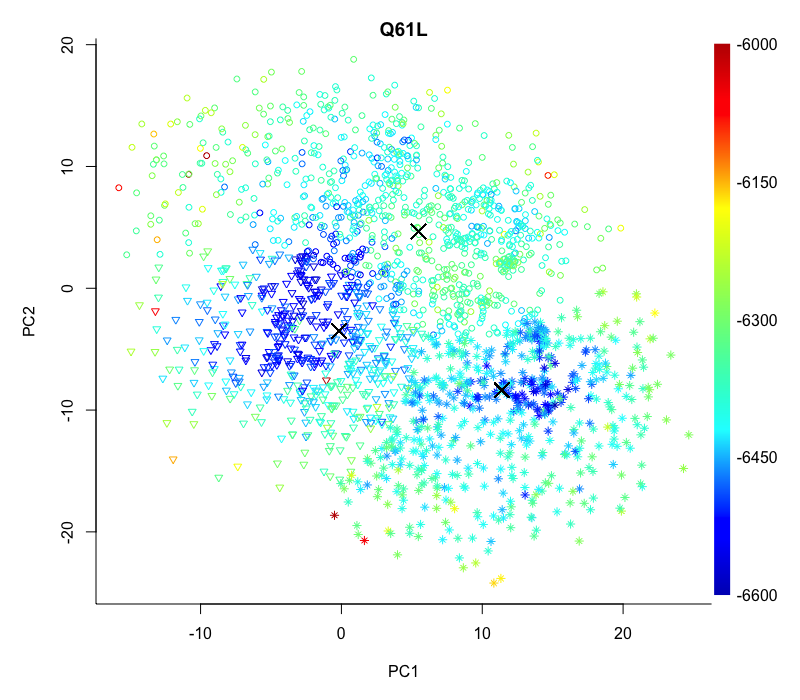}\\[-3mm]
\end{tabular}
\caption{Each panel shows 2000 structures (as symbols \ding{83}, $\triangledown$, and $\circ$) of the human H-Ras protein, and each type of these symbols represents a cluster (basin) obtained with the proposed algorithm, for which bandwidths are selected using the gradient-based LSCV method proposed in Section~\ref{sec:bwselection}. The symbols $\times$ are the local modes. To aid visualization, structures are embedded onto the top two principal components obtained via Principal Component Analysis. Symbols are color-coded by their potential energies in a red-to-blue color scheme showing high-to-low energies. The left panel organizes structures accessed under physiological conditions by WT H-Ras; the right panel does so for the mutated, oncogenic Q61L H-Ras.} 
\label{fig:BioApplication}
\end{figure}

The proposed algorithm is used to group the structures accessed by each H-Ras variant into local antimodes (to which we refer as energy basins). Basins with many low-energy structures (blue dots in Figure~\ref{fig:BioApplication}) correspond to stable and semi-stable structural states. The left panel in Figure~\ref{fig:BioApplication} shows two such basins, one on the top left and one on the bottom right. More low-energy structures are contained in the basin shown in the top left, which indicates this is a wider and so more stable basin. Blue dots are found in between the basins, which indicate that the protein can transition between the two basins via low-energy structures; that is, an energetically feasible pathways exists to regulate the transition between the basins. Knowledge of the transition between on and off states for WT H-Ras allows us to speculate that the basins correspond to such states, as revealed by the proposed algorithm.   

A comparison with the right panel in Figure~\ref{fig:BioApplication}, which shows the organization for Q61L H-Ras, shows two major differences with WT H-Ras. First, both basins become narrower; that is, they contain fewer low-energy structures. This suggests that the mutation impacts the structural plasticity of H-Ras. Second, few to no low-energy structures can be found between the basins, which suggests that the energetic pathway between the basins becomes more energetically costly. This in turn suggests that the Q61L mutation directly impacts the transition between the on and off states and so rigifies H-Ras. Such information is precious, as it allows formulating a detailed structure-based hypothesis that links sequence mutations to dysfunction via changes to structure and structure dynamics.

In addition to the two main basins, a third one is evident in the left panel of Figure~\ref{fig:BioApplication} for the WT H-Ras. This contains fewer structures and is shallower. The right panel of Figure~\ref{fig:BioApplication} indicates that the basin becomes even shallower in Q61L H-Ras. These results are in great agreement with early work in Clausen et al. (2015), where a Conf1 basin was suggested to exist in WT H-RAS and correspond to an unanticipated structural state. Specifically, by analyzing crystallographic structures whose projections over PC1 and PC2 fell on this basin, work in Clausen et al. (2015) suggested that this smaller and shallower basin corresponded to a structural state that was an intermediate between the known on and off states between the GTP- and GDP-bound states of WT H-Ras. In strong agreement with the results presented here, work in Clausen et al. (2015) additionally reported that this basin all but disappeared in Q61L.

More broadly, the shown application suggests that by organizing an energy landscape into the major local antimodes, the proposed algorithm allows understanding in great detail the impact of a mutation on the structural basin-to-basin dynamics that characterizes flexible biomolecules, such as proteins, and even obtaining an explanation for dysfunction in terms of changes to the underlying energy landscape and the dynamics on it.

\subsubsection{Spatial clustering of malaria episodes}

We applied our regression mean shift algorithm to a malaria episodes dataset available in the \proglang{R} package \pkg{SPODT} (Gaudart et al, 2015) and obtained a spatial clustering result, as shown in Figure~\ref{fig:Malaria}. The dataset contains 168 observations, each corresponding to the longitudinal and latitudinal coordinates of a household, and the mean value of the number of malaria episodes per child in the household in Bandiagara, Mali, from November to December 2009. Our algorithm returns three clusters using the automatically selected bandwidth. The estimated modes represent high-risk locations and different clusters are separated by low-risk valleys. 

As a comparison, we also show the partitioning result of the CART algorithm (Breiman et al., 1993) in Figure~\ref{fig:Malaria}. The same dataset has also been analyzed using a variant of CART algorithm called spatial oblique decision tree (SpODT). See Gaudart et al. (2015). The shape of clusters found using our regression MS algorithm appear different from that obtained from CART and its variant, which reflects the fundamental difference in the ideas of partitioning: the mathematical models behind the clusters in our regression MS are the ascending manifolds of the regression functions (see Section~\ref{sec:partition}), while CART and its variants can be viewed as piecewise constant approximation of the regression function through their leaf nodes.

\begin{figure}[h]
\centering
\includegraphics[width=0.75\columnwidth]{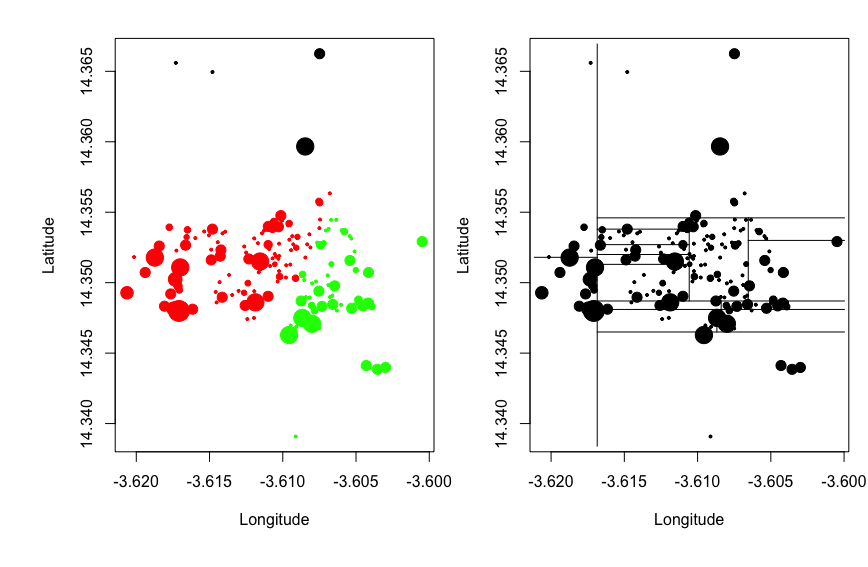} 
\caption{The graphs shows the spatial partitioning results for the malaria episodes dataset available in the \proglang{R} package \pkg{SPODT}, using our regression mean shift algorithm (left panel), and CART (right panel). The mean value of the malaria episodes at each location is represented by the size of the dots. The bandwidth in regression MS algorithm is selected using the method as given in Section~\ref{sec:bwselection}, and the three clusters are represented by different colors (black, red, and green).}
\label{fig:Malaria}
\end{figure}

\section{Discussions}\label{sec:discussion}
In this paper we develop a regression mean shift algorithm to partition the input space and estimate the local modes of the regression function. The algorithm is shown to be convergent and we give the non-asymptotic rates of convergence for the local mode estimators. A bandwidth selection strategy is discussed and is shown to be effective in simulations and real data applications.

The idea of using regression MS to find local modes can be extended to extract ridges of regression functions. Ridges are low-dimensional geometric features where the function values are local maximum in a subspace, which generalizes the concepts of local modes and can be used to model filamentary structures. An algorithm called Subspace Constrained Mean Shift (SCMS) was developed in Ozertem and Erdogmus (2011) to extract ridges of KDEs. Some theoretical analysis of this algorithm can be found in Genovese et al. (2014) and Qiao and Polonik (2016). It is straightforward to convert our regression mean shift to its subspace constrained version as follows. 
%
Let $\wh v_{1,*}(x),\cdots,\wh v_{d,*}(x)$ be the orthonormal eigenvectors of $\nabla^2 \wh r_*(x)$ associated with eigenvalues $\wh \lambda_{1,*}(x)\geq\cdots\geq\wh \lambda_{d,*}(x)$. For a fixed $s\in\{2,\cdots,d\}$, denote $\wh V_*(x) = [\wh v_{s,*}(x),\cdots,\wh v_{d,*}(x)]$. For any starting point $z_0\in\mathcal{X}$, define $z_1,z_2,\cdots,$ iteratively by
\begin{align}\label{gms}
z_{j+1} = \wh V_*(z_j)\wh V_*(z_j)^\top \wh m_*(z_j) + z_j, \; j=1,2,\cdots.
\end{align}
The limit of the sequence is considered as an estimate of a $s$-ridge point of the regression function. We leave the analysis of this SCMS algorithm as a future work.

\section{Proofs}\label{sec:proofs}
This section contain the proofs of theoretical results in Sections~\ref{sec:algorithm} and \ref{sec:theory}. Note that the proof of Theorem 2.1 is very similar to that of Theorem 1 in Ghassabeh (2015) and is hence omitted. In the proofs we use $C$ to denote a constant that may change its value depending on where it appears.

\subsection{Proof of Lemma~\ref{gradlemma}}
\begin{proof}
Using the expression in (\ref{MMest}), we have
\begin{align*}
\wh r_*(z_{j+1}) - \wh r_*(z_{j}) 
= \frac{c_{k,d}}{nh^{d+2}} \sum_{i=1}^n \frac{\wt Y_i}{\wh f(X_i)} \Big[k\Big(\|\frac{z_{j+1}-X_i}{h}\|^2\Big)- k\Big(\|\frac{z_{j}-X_i}{h}\|^2\Big)\Big].
\end{align*}
The convexity assumption of $k$ implies that $k(x_2)-k(x_1)\geq g(x_1)(x_1-x_2)$ for all $x_1,x_2\in[0,\infty)$ and $x_1\neq x_2$. Then using (\ref{wixdef}) we have
\begin{align*}
&\wh r_*(z_{j+1}) - \wh r_*(z_{j}) \nonumber \\
\geq & 
%
\frac{c_{k,d}}{nh^{d+2}}\sum_{i=1}^n \wt Y_i w_i(z_j) [2(z_{j+1}-z_j)^TX_i + \|z_{j}\|^2 -\|z_{j+1}\|^2 ] \nonumber\\
=& \frac{c_{k,d}}{nh^{d+2}} \Big[2(z_{j+1} - z_j)^T \sum_{i=1}^n \wt Y_i w_i(z_j) X_i + (\|z_{j}\|^2 - \|z_{j+1}\|^2)\sum_{i=1}^n \wt Y_i w_i(z_j) \Big] \nonumber \\
= &  \frac{c_{k,d}}{nh^{d+2}} \|z_{j+1} - z_j\|^2 \sum_{i=1}^n \wt Y_i w_i(z_j) \nonumber \\
\geq & \frac{c_{k,d}}{nh^{d+2}} \|z_{j+1} - z_j\|^2 \inf_{z\in\mathcal{C}} \sum_{i=1}^n \wt Y_i w_i(z),
\end{align*}
where $\mathcal{C}$ is the convex hull of $\{X_1,\cdots,X_n\}$. Notice that $\inf_{z\in\mathcal{C}} \sum_{i=1}^n \wt Y_i w_i(z) >0$, which implies that $\wh r_*(z_{j+1}) - \wh r_*(z_{j})>0$ as long as $z_{j+1}\neq z_{j}$. Since $\wh r$ is upper bounded, the sequence $\wh r(z_{j})$ converges, and it follows that $\|z_{j+1} - z_{j}\|\rightarrow 0$ as $j\rightarrow\infty$. Since $\wh m_*(z_j)= z_{j+1} - z_j$, using (\ref{grdexp}) we then get $\nabla\wh r_*(z_j)\rightarrow0.$
\end{proof}

\subsection{Proof of Lemma~\ref{morse}}
\begin{proof}
Let $\xi^\prime$ and $\xi^{\prime\prime}$ be the first two derivatives of $\xi$, respectively, and define
\begin{align}
\rho_1(x) = \mathbb{E}  \xi^\prime(r(x) +\epsilon_1), \text{ and }
\rho_2(x) = \mathbb{E}  \xi^{\prime\prime}(r(x) +\epsilon_1).
\end{align}
We have that $\rho_1(x)>0$, and both $\rho_1(x)$ and $\rho_2(x)$ are bounded. Note that for $\alpha\in\mathbb{N}^d$ with $|\alpha| =1$,
\begin{align}\label{wtr1}
\partial^\alpha \wt r(x) = \rho_1(x)\partial^\alpha r(x) ,
\end{align}
and for $\alpha\in\mathbb{N}^d$ with $|\alpha| =2$ such that $\alpha = \alpha_1 + \alpha_2$, where $|\alpha_1| = |\alpha_2| =1,$
\begin{align}
\partial^\alpha \wt r(x) =\rho_1(x)\partial^\alpha r(x) + \rho_2(x)\partial^{\alpha_1} r(x)\partial^{\alpha_2} r(x).
\end{align}
In the matrix form, we get 
\begin{align*}
&\nabla \wt r(x) = \rho_1(x) \nabla r(x),\\
&\nabla^2 \wt r(x) = \rho_1(x) \nabla^2 r(x) + \rho_2(x) \nabla r(x) [\nabla r(x)]^T.
\end{align*}
Hence $\nabla^2 \wt r(x) = \rho_1(x) \nabla^2 r(x)$ for all $x$ such that $\nabla \wt r(x)=0$. Since $\rho_1(x)>0$, the critical points of $\wt r(x)$ and $r(x)$ are the same with the same indices. 

When $r$ is a Morse function, the above analysis implies that $\wt r$ is also a Morse function. To show that the ascending manifolds of $r$ and $\wt r$ are the same, we will prove that the trajectories of integral curves driven by $\nabla r$ and $\nabla \wt r$ are the same when the starting points are the same. To this end, we show that there exists a reparameterization function $\eta:\mathbb{R}_{\geq0}\rightarrow\mathbb{R}_{\geq0}$ such that $\wt\phi_x(t)=\phi_x(\eta(t)),t\geq0$, where $\phi_x$ and $\wt\phi_x$ are integral curves driven by $\nabla r$ and $\nabla \wt r=\rho_1\nabla r$, respectively, defined as the solutions of
\begin{align*}
&\phi_x^\prime(t) = \nabla r(\phi_x(t)), \;t\geq 0;\; \phi_x(0) = x;\\
&\wt\phi_x^\prime(t) = \nabla \wt r(\wt \phi_x(t)),\; t\geq 0;\; \wt\phi_x(0) = x.
\end{align*}
Here $\eta$ is the solution of the ODE $\eta^\prime (t) = \rho_1(\phi_x(\eta(t)));\; \eta(0)=0$. Then we have $\wt\phi_x(0)=\phi_x(\eta(0))=x$ and for $t\geq 0,$
\begin{align*}
(\phi_x\circ \eta)^\prime(t)= \phi_x^\prime(\eta(t)) \; \eta^\prime(t) = \nabla r(\phi_x\circ\eta(t)) \, \rho_1( \phi_x\circ \eta(t)) =  \nabla \wt r(\phi_x\circ \eta (t)).
\end{align*}
Hence $\wt\phi_x(t)=\phi_x(\eta(t)),t\geq0$. So the conclusion of this lemma follows.
\end{proof}

\subsection{Proof of Theorem~\ref{derbound}}

 We use empirical process theory in the proofs. Let $\mathbb{P}$ be the probability measure of $(X,Y)$, and $\mathbb{P}_n$ be the empirical probability measure with respect to $\{(X_i,Y_i): i=1,\cdots,n\}$ such that we write $\mathbb{P}(g)=\mathbb{E}g(X,Y)$, and $\mathbb{P}_n(g) = \frac{1}{n}\sum_{i=1}^n g(X_i,Y_i)$, for any measurable function $g: \mathbb{R}^d \times\mathbb{R} \rightarrow\mathbb{R}$. Let 
\begin{align}
\mathbb{G}_n(g) = \frac{1}{\sqrt{n}} \sum_{i=1}^n [\mathbb{P}_n(g) - \mathbb{P}(g) ].
\end{align}

Let $\mathcal{G}$ be a set of measurable functions from $\mathbb{R}^{d+1}$ to $\mathbb{R}$. $\mathcal{G}$ is called a uniformly bounded VC-class if there exists a constant $B> 0$ such that $\sup_{x\in\mathbb{R}^{d+1}}|g(x)|\leq B$ for all $g\in \mathcal{G}$, and there exist positive numbers $A$ and $\nu$ (called the characteristics) such that for every $0<\epsilon\leq B$, 
\begin{align*}
\sup_Q \mathcal{N}(\mathcal{G}, L_2(Q),\epsilon) \leq \Big( \frac{AB}{\epsilon}\Big)^\nu,
\end{align*}
where the covering number $\mathcal{N}(\mathcal{G}, L_2(Q),\epsilon)$ denotes the smallest number of $L_2(Q)$-balls of radius at most $\epsilon$ needed to cover $\mathcal{G}$ and the supremum is taken over all probability measures $Q$ on $\mathbb{R}^{d+1}$. 

The following proposition generalizes Proposition A.5 in Sriperumbudur and Steinwart (2012), based on Talagrand's inequality. See Gin\'{e} and Guillou (2002) and Einmahl and Mason (2000). Its proof is given in the appendix.
\begin{proposition}\label{talagrand}
Let $M$ is a real-valued function on $\mathbb{R}^d$ with bounded support $\mathcal{S}$ such that $M \in L_\infty(\mathbb{R}^d)\cap L_2(\mathbb{R}^d)$. Suppose that the marginal density $f$ of $X$ is uniformly bounded on $\mathcal{X}^{\eta h_0}$ for some constant $h_0>0$, where $\eta=\sup_{x\in \mathcal{S}}\|x\|$. Suppose that
\begin{align}\label{fclass}
\mathscr{F}:=\{\mathbb{R}^d\times\mathbb{R} \ni(u,v) \mapsto M(x-u) : \; x\in \mathbb{R}^d\}
\end{align}
is a uniformly bounded VC-class with characteristics $(V,\nu)$. For $h>0$, let $\zeta_h:\mathcal{X}\times\mathbb{R} \rightarrow\mathbb{R}$ a function indexed by $h$ such that $\sup_{0<h\leq h_0}\sup_{x,y\in\mathcal{X}\times \mathbb{R}}|\zeta_h(x,y)|\leq L$ for some constant $L\in(0,\infty)$.  For $(u,v)\in\mathcal{X}\times \mathbb{R}$ and $h>0$, denote $g_{x,h}(u,v)=\frac{1}{h^d}\zeta_h(u,v)M((x-u)/h)$ and $\mathcal{G}_h = \{g_{x,h}(\cdot): x\in\mathcal{X}\}$. Then, there exists a positive constant $C$ only depending on $L$, $M$, $f$, $A$ and $\nu$ such that, for all $n \geq 1$, $0<h < h_0$, and $\tau>0$ we have
\begin{align*}
\mathbb{P}^n\Big(\frac{1}{\sqrt{n}}\sup_{g\in \mathcal{G}_h } |\mathbb{G}_n(g)| \leq \frac{C}{nh^d}\log\frac{C}{h} + \sqrt{\frac{C}{nh^d}\log\frac{C}{h}} + \frac{\tau C}{nh^d} + \frac{C\sqrt{\tau}}{\sqrt{nh^d}}\Big) \geq 1 - e^{-\tau}.
\end{align*}
\end{proposition}
%

The proof of Theorem~\ref{derbound} needs an intermediate estimator, as define below. Denote $f_h=\mathbb{E} \wh f$. Let $$\wh r_0(x) = \frac{1}{nh^d} \sum_{i=1}^n \frac{\wt Y_iK_h(x-X_i) }{f_h(X_i)},\; x\in\mathcal{X}. $$ Note that
\begin{align}
\partial^\alpha\wh r_0(x) = \frac{1}{nh^{d+|\alpha|}}\sum_{i=1}^n \frac{\wt Y_i \partial^\alpha K((x-X_i)/h)}{f_h(X_i)}.
\end{align} 
Then for $x\in\mathcal{X}$ we can write $\partial^\alpha\wh r_*(x) - \partial^\alpha \wt r(x) =I_n(x) + II_n(x)  + III_n(x),$ where
%
\begin{align*}
&I_n(x) =  \partial^\alpha\wh r_*(x) - \partial^\alpha\wh r_0(x),\\
&II_n(x)  = \partial^\alpha\wh r_0(x) - \mathbb{E} \partial^\alpha\wh r_0(x),\\
& III_n(x) = \mathbb{E} \partial^\alpha\wh r_0(x) - \partial^\alpha \wt r(x).
\end{align*}
The conclusion in Theorem~\ref{derbound} is a direct consequence of Propositions~\ref{Inres}, \ref{IInres}, \ref{IIInres} given in the sequel, which are used to analyze $\sup_{x\in\mathcal{X}}|I_n(x)|$, $\sup_{x\in\mathcal{X}}|II_n(x)|$, and $\sup_{x\in\mathcal{X}}|III_n(x)|$, respectively. In particular, $\sup_{x\in\mathcal{X}}|I_n(x)|$, $\sup_{x\in\mathcal{X}}|II_n(x)|$ are stochastic terms that are analyzed using Proposition~\ref{talagrand}. We first consider $\sup_{x\in\mathcal{X}}|I_n(x)|$.

\begin{proposition}\label{Inres}
Under the same assumptions as in Theorem~\ref{derbound}, there exist constants $C>0$, $c>0,$ and $h_0>0$ such that for all $|\alpha|\leq 2$, $n\geq 1$, $0<h\leq h_0$, $\tau>1$ satisfying $nh^d\geq c(\tau \vee |\log h|)$ we have 
\begin{align}
\mathbb{P}^n \Big(\sup_{x\in\mathcal{X}}|\partial^\alpha\wh r_*(x) - \partial^\alpha \wh r_0(x)| < C\sqrt{\tau \vee |\log h|} \gamma_{n,h}^{(|\alpha|)} \Big) \geq 1 - 2e^{-\tau}.
\end{align}
\end{proposition}

\begin{proof}
Assume $h\leq\delta$. Using Taylor expansion and the assumption that $K$ is spherically symmetric with its support contained in the unit ball, we have
\begin{align}\label{kdebias}
&\sup_{x\in \mathcal{X}} |f_h(x) - f(x)| \nonumber\\
= & \sup_{x\in \mathcal{X}} \Big|\int K(u)f(x-hu)du - f(x)\Big| \nonumber\\
\leq & h^2  d \int K(u) \|u\|^2du \sup_{x\in \mathcal{X}^\delta} \max_{|\alpha|=2}|\partial^\alpha f(x)|.
\end{align}
Hence there exists $h_0\in(0,\delta]$ such that for all $0<h\leq h_0$,
$$\sup_{x\in \mathcal{X}} |f_h(x) - f(x)| \leq \frac{1}{3} \varepsilon_0,$$ where $0<\varepsilon_0\leq \inf_{x\in\mathcal{X}}f(x)$ is given in assumption \textbf{A1}. This implies that $\inf_{x\in \mathcal{X}} f_h(x)\geq \frac{2}{3} \varepsilon_0$. Below we always assume that $0<h\leq h_0$. Notice that 
\begin{align}\label{rstardecomp}
|\partial^\alpha\wh r_*(x) - \partial^\alpha\wh r_0(x)| \leq &\sup_{x\in \mathcal{X}} |s_n(x)| \sup_{x\in \mathcal{X}}\wh r_+^\alpha(x),
\end{align}
where $s_n(x)=[\wh f(x)]^{-1} - [f_h(x)]^{-1}$ and
\begin{align}
\wh r_+^\alpha(x) = \frac{1}{nh^{d+|\alpha|}}\sum_{i=1}^n|\wt Y_i \partial^\alpha K_h(x-X_i)|.
\end{align}
Notice that
\begin{align}
s_n(X_i) = \frac{f_h(X_i) - \wh f(X_i)}{f_h(X_i)^2} + \frac{(f_h(X_i) - \wh f(X_i))^2}{f_h(X_i)^2\wh f(X_i)}.
\end{align}
For $|\alpha|\leq 2$, consider the class of functions $\mathcal{K}_\alpha=\{\partial^\alpha K(x - \cdot): \; x\in\mathbb{R}^d\}.$ Then $\mathcal{K}_\alpha$ is uniformly bounded under assumption \textbf{K}. Note that $K(x - \cdot)=k(\|x-\cdot\|^2)$, where $k$ and its first two derivatives have bounded variation. It is known from Nolan and Pollard (1987) that in general $\mathscr{F}$ in (\ref{fclass}) is a VC-class if $M(x)=\phi(p(x))$, where $p$ is a polynomial and $\phi$ is a bounded real function of bounded variation. When $|\alpha|=0$, it is clear that $\mathcal{K}_\alpha$ is a VC-class. When $|\alpha|=1$, we have $\partial^\alpha K(x - \cdot)=k^\prime(\|x-\cdot\|^2)[2\alpha^T(x-\cdot)]$. Noice that both $\{2\alpha^T(x-\cdot): \;x\in\mathbb{R}^d\}$ and $\{k^\prime(\|x-\cdot\|^2): \;x\in\mathbb{R}^d\}$ are VC-classes. We then apply Lemma A.6 in Chernozhukov et al. (2013) to conclude that $\mathcal{K}_\alpha$ is also a VC-class. A similar argument also applies to $|\alpha|=2$.

For $u\in \mathbb{R}^d$, let $g_{x,h}(u) = \frac{1}{h^d} K((u-x)/h)$ and $\mathcal{G}_h=\{g_{x,h}(\cdot): x\in\mathcal{X}\}$. Then notice that $\sup_{x\in \mathcal{X}} |\wh f(x) - \mathbb{E}\wh f(x)| = \frac{1}{\sqrt{n}} \sup_{g\in\mathcal{G}_h}|\mathbb{G}_n(g)|$. Applying Proposition~\ref{talagrand} we get that there exists a constant $C_0>0$ such that for all $n\geq 1$, $h\in(0,1)$, and $\tau>1$ satisfying $nh^d\geq \tau$ and $nh^d\geq |\log h|$, with probability at least $1 - e^{-\tau}$,
\begin{align}\label{kdestoc}
\sup_{x\in \mathcal{X}} |\wh f(x) - f_h(x)| < C_0\sqrt{\tau \vee |\log h|} \gamma_{n,h}^{(0)}.
\end{align}
Suppose that $C_0\sqrt{\tau \vee |\log h|} \gamma_{n,h}^{(0)} < \frac{1}{3} \varepsilon_0$. On the event in (\ref{kdestoc}), we have $\sup_{x\in\mathcal{X}}|\wh f(x) - f_h(x)|< \frac{1}{3}\varepsilon_0$ and $\inf_{x\in \mathcal{X}} \wh f(x)\geq \frac{1}{3} \varepsilon_0$. Therefore $\sup_{x\in \mathcal{X}} |s_n(x)| \leq 5\varepsilon_0^{-2} \sup_{x\in \mathcal{X}} |\wh f(x) - f_h(x)|$ and with probability at least $1 - e^{-\tau}$,
\begin{align}\label{snbound}
\sup_{x\in \mathcal{X}} |s_n(x)| < 5\varepsilon_0^{-2}C_0\sqrt{\tau \vee |\log h|} \gamma_{n,h}^{(0)} .
\end{align}
Here for all $x\in\mathcal{X}$,
\begin{align}\label{rplusbound1}
0\leq & \mathbb{E} \wh r_+^\alpha(x) \nonumber\\
= & \frac{1}{h^{d+|\alpha|}}  \mathbb{E} \int_{\mathbb{R}^d} \Big|\xi(r(x)+\epsilon_1) \partial^\alpha K\Big(\frac{x-u}{h}\Big)\Big| f(u) du \nonumber\\
= & \frac{1}{h^{|\alpha|}} \mathbb{E} \int_{\mathbb{R}^d} |\xi(r(x)+\epsilon_1) \partial^\alpha K(w)| f(x-hw) dw  \nonumber\\
\leq & \frac{1}{h^{|\alpha|}} C_u \|\partial^\alpha K\|_1 \sup_{x\in \mathcal{X}^\delta} f(x) =: C_1 \frac{1}{h^{|\alpha|}},
\end{align}
where $\|\cdot\|_1$ is the $L_1$ norm.

For $(u,v)\in \mathcal{X} \times\mathbb{R}$, let $g_{x,h}^\alpha(u,v) = \frac{1}{h^d}|\xi(v)\partial^\alpha K((u-x)/h)|$. Then notice that we can write $\wh r_+^\alpha(x) - \mathbb{E}\wh r_+^\alpha(x) = \frac{1}{\sqrt{n}h^{|\alpha|}} \mathbb{G}_n(g_{x,h}^\alpha)$ and so that
\begin{align*}
\sup_{x\in\mathcal{X}}|\wh r_+^\alpha(x) - \mathbb{E}\wh r_+^\alpha(x)| = \frac{1}{\sqrt{n}h^{|\alpha|}} \sup_{g\in \mathcal{G}_{h,\alpha}}|\mathbb{G}_n(g)|,
\end{align*}
where $\mathcal{G}_{h,\alpha} =\{g_{x,h}^\alpha(\cdot): x\in \mathcal{X}\}.$ Applying Proposition~\ref{talagrand} we get that for $n\geq 1$, $h\in(0,1)$, and $\tau>1$ satisfying $nh^d\geq \tau$ and $nh^d\geq |\log h|$,
\begin{align}\label{rplusbound2}
&\mathbb{P}^n \Big(\sup_{x\in\mathcal{X}}|\wh r_+^\alpha(x) - \mathbb{E}\wh r_+^\alpha(x)| < C_2\sqrt{\tau \vee |\log h|} \gamma_{n,h}^{(|\alpha|)} \Big) \geq 1 - e^{-\tau},
\end{align}
for some constant $C_2>0$. So it follows from (\ref{rplusbound1}) and (\ref{rplusbound2}) that for $C_3=C_1\vee C_2$, 
\begin{align}\label{rplusbound3}
&\mathbb{P}^n \Big(\sup_{x\in\mathcal{X}}\wh r_+^\alpha(x)  < C_3( h^{-|\alpha|} + \sqrt{\tau \vee |\log h|} \gamma_{n,h}^{(|\alpha|)} ) \Big) \geq 1 - e^{-\tau}.
\end{align}
Combing (\ref{rstardecomp}), (\ref{snbound}) and (\ref{rplusbound3}), we then get the conclusion of this proposition.
\end{proof}

Next we consider $\sup_{x\in\mathcal{X}}|II_n(x)|$. 
\begin{proposition}\label{IInres}
Under the same assumptions as in Theorem~\ref{derbound}, there exist constants $C>0$, $c>0$, and $h_0>0$ such that for all $|\alpha|\leq 2$, $n\geq 1$, $0<h\leq h_0$, $\tau>1$ satisfying $nh^d\geq c(\tau \vee |\log h|)$ we have 
\begin{align}\label{r0bound2}
\mathbb{P}^n \Big(\sup_{x\in\mathcal{X}}|\partial^\alpha\wh r_0(x) - \mathbb{E}\partial^\alpha\wh r_0(x)| < C\sqrt{\tau \vee |\log h|} \gamma_{n,h}^{(|\alpha|)} \Big)  \geq 1 - e^{-\tau}.
\end{align}
\end{proposition}

\begin{proof}
For $(u,v)\in \mathcal{X} \times\mathbb{R}$, let $\zeta_h(u,v)=\frac{\xi(v)}{f_h(u)}$ and $p_{x,h}^\alpha(u,v) = \frac{1}{h^d}\zeta_h(u,v)\partial^\alpha K((x-u)/h)$. Then notice that we can write $\partial^\alpha\wh r_0(x) - \mathbb{E}\partial^\alpha\wh r_0(x) = \frac{1}{\sqrt{n}h^{|\alpha|}} \mathbb{G}_n(p_{x,h}^\alpha)$ and so that
\begin{align*}
\sup_{x\in\mathcal{X}}|\partial^\alpha\wh r_0(x) - \mathbb{E}\partial^\alpha\wh r_0(x)| = \frac{1}{\sqrt{n}h^{|\alpha|}} \sup_{p\in \mathcal{P}_{h,\alpha}}|\mathbb{G}_n(p)|.
\end{align*}
where $\mathcal{P}_{h,\alpha} =\{p_{x,h}^\alpha(\cdot): x\in \mathcal{X}\}.$ Note that using the same $h_0$ in the proof of Proposition~\ref{rstardecomp}, we have $\sup_{u\in\mathcal{X}}\sup_{v\in\mathbb{R}}|\zeta_h(u,v)| \leq 2\varepsilon_0^{-1} C_u.$ Applying Proposition~\ref{talagrand} we then get (\ref{r0bound2}).
%
\end{proof}

Next we consider $\sup_{x\in\mathcal{X}}|III_n(x)|$.
\begin{proposition}\label{IIInres}
Under the same assumptions as in Theorem~\ref{derbound}, there exist constants $C>0$ and $h_0>0$ such that for all $0<h\leq h_0$, and $|\alpha|\leq 2$ we have 
\begin{align}
\sup_{x\in\mathcal{X}} |\mathbb{E} \partial^\alpha\wh r_0(x) - \partial^\alpha \wt r(x)| \leq Ch^{(3-|\alpha|)\wedge 2}.
\end{align}
\end{proposition}

\begin{proof}
Let $c_f=\sup_{x\in\mathcal{X}^\delta}\sup_{|\beta|\leq 3}|\partial^\beta f(x)|$ and $c_K=[\max\{ \int_{\mathbb{R}^d} K(u) \|u\|^jdu: j=1,2\}]\vee 1$. Let $c_{\xi}=[\sup_{x\in\mathbb{R}}\max(\xi(x),\xi^\prime(x),\xi^{\prime\prime}(x),\xi^{\prime\prime\prime}(x))]\vee 1$ and $c_r=[\sup_{x\in\mathcal{X}^\delta}\sup_{|\beta|\leq 3}|\partial^\beta r(x)|]\vee1$. 

Below we take $0<h\leq \frac{1}{2}\delta$ so that $(\mathcal{X}^h)^h\subset \mathcal{X}^\delta$. For $|\alpha|=0,1$, using a Taylor expansion of order 2 and assumption {\bf K}, we have
\begin{align}\label{kdebias2}
&\sup_{x\in \mathcal{X}^h} | \partial^\alpha f_h(x) - \partial^\alpha f(x)| \nonumber\\
= & \sup_{x\in \mathcal{X}^h} |\int_{\mathbb{R}^d} K(u)\partial^\alpha f(x-hu)du - \partial^\alpha f(x)| \nonumber\\
\leq & h^2  \int_{\mathbb{R}^d} K(u) \|u\|^2du \sup_{x\in \mathcal{X}^\delta} \max_{|\beta|=|\alpha|+2}|\partial^\beta f(x)| \nonumber\\
\leq & d c_f c_K h^2.
\end{align}
Similarly for $|\alpha|=2$, using a Taylor expansion of order 1 we have
\begin{align}\label{kdebias3}
& \sup_{x\in \mathcal{X}^h} |\partial^\alpha f_h(x) - \partial^\alpha f(x)| \nonumber \\
= & \sup_{x\in \mathcal{X}^h} |\int_{\mathbb{R}^d} K(u)\partial^\alpha f(x-hu)du - \partial^\alpha f(x)| \nonumber\\
\leq & h  \int K(u) \|u\|du \sup_{x\in \mathcal{X}^\delta} \max_{|\beta|=3}|\partial^\beta f(x)|\nonumber\\
\leq &\sqrt{d} c_f c_K h.
\end{align}

Let $\eta_0=\sup\{\eta\in(0,\delta]:\; \inf_{x\in\mathcal{X}^\eta} f(x)\geq \frac{1}{2}\varepsilon_0\}$. Under assumptions {\bf A1} and {\bf A2}, we have $\eta_0>0$. Let $h_0=\min\{\eta_0, \frac{1}{2}\delta,  (\sqrt{d}c_K)^{-1}, (\frac{\varepsilon_0}{2dc_fc_K})^{1/2}\}$ and we take $h\in(0,h_0]$ below. Using (\ref{kdebias2}) and (\ref{kdebias3}) we have $\sup_{x\in \mathcal{X}^h} |f(x) - f_h(x)|\leq c_f\wedge (\frac{1}{2}\varepsilon_0)$ and $\sup_{x\in \mathcal{X}^h} |\partial^\alpha f_h(x) - \partial^\alpha f(x)|\leq c_f$ for all $|\alpha|=1,2$. This implies that for all $x\in\mathcal{X}^h$ we have $\frac{1}{2}\varepsilon_0\leq f_h(x)\leq 2c_f$, and $|\partial^\alpha f_h(x)|\leq 2c_f$ for all $|\alpha|=1,2.$
%

Let $q_h(x) = \frac{f(x)}{f_h(x)}.$ Using (\ref{kdebias2}) we have for all $x\in\mathcal{X}^h$,
\begin{align}\label{qhrate0}
|q_h(x)-1|\leq \frac{2c_fc_Kh^2}{\varepsilon_0} := C_{q,0}h^2.
\end{align}
When $|\alpha|=1$, 
\begin{align*}
\partial^\alpha q_h(x) = \frac{\partial^\alpha f(x) - \partial^\alpha f_h(x) }{f_h(x)} - \frac{[f(x) -f_h(x)] \partial^\alpha f_h(x) }{f_h(x)^2}.
\end{align*}
Hence using (\ref{kdebias2}) we have that for all $x\in\mathcal{X}^h$,
\begin{align}\label{qhrate1}
|\partial^\alpha q_h(x)| \leq \frac{2c_fc_Kh^2}{\varepsilon_0} + \frac{8c_f^2c_Kh^2}{\varepsilon_0^2} := C_{q,1}h^2.
\end{align}
When $|\alpha|=2$, suppose that $\alpha=\alpha_1+\alpha_2$, where $|\alpha_1|=|\alpha_2|=1$. We have
\begin{align}
\partial^\alpha q_h(x) = &\frac{\partial^\alpha f(x) - \partial^\alpha f_h(x) }{f_h(x)} - \frac{[\partial^{\alpha_1} f(x) -\partial^{\alpha_1} f_h(x)] \partial^{\alpha_2} f_h(x) + [\partial^{\alpha_2} f(x) -\partial^{\alpha_2} f_h(x)] \partial^{\alpha_1} f_h(x)}{[f_h(x)]^2}\nonumber\\
& + [f_h(x) - f(x)]\Big[\frac{2\partial^{\alpha_1} f_h(x) \partial^{\alpha_2} f_h(x) }{f_h(x)^4} - \frac{\partial^\alpha f_h(x)}{f_h(x)^2}\Big].\nonumber
\end{align}
Hence using (\ref{kdebias2}) and (\ref{kdebias3}) we have that for all $x\in\mathcal{X}^h$, 
\begin{align}\label{qhrate2}
|\partial^\alpha q_h(x)| &\leq \frac{2c_fc_Kh}{\varepsilon_0} + \frac{24c_f^2c_Kh^2}{\varepsilon_0^2} + \frac{128c_f^3c_Kh^2}{\varepsilon_0^4} \nonumber\\
&\leq \frac{2c_fc_Kh}{\varepsilon_0} + \frac{12c_f^2c_K\delta h}{\varepsilon_0^2} + \frac{64c_f^3c_K\delta h}{\varepsilon_0^4} := C_{q,2}h.
\end{align}
We can write for $|\alpha|=0,1,2,$
\begin{align}\label{expalphar}
&\mathbb{E} \partial^\alpha \wh r_0(x) \nonumber\\
= & \frac{1}{h^{d+|\alpha|}} \mathbb{E} \frac{\wt Y_1\partial^\alpha K(\frac{x-X_1}{h})}{f_h(X_1)} \nonumber\\
= & \frac{1}{h^{d+|\alpha|}} \mathbb{E} \frac{\xi(r(X_1)+\epsilon_1) \partial^\alpha K(\frac{x-X_1}{h})}{f_h(X_1)} \nonumber\\
=& \frac{1}{h^{d+|\alpha|}} \mathbb{E} \int_{\mathbb{R}^d} \xi(r(u)+\epsilon_1) \partial^\alpha K\Big(\frac{x-u}{h}\Big)q_h(u)du \nonumber\\
= & \frac{1}{h^{|\alpha|}} \mathbb{E}  \int_{\mathbb{R}^d} \xi(r(x-hw) +\epsilon_1) q_h(x-hw)\partial^\alpha K(w)dw \nonumber\\
= & \frac{1}{h^{|\alpha|}}\mathbb{E}  \int_{\mathbb{R}^d}  \partial_w^\alpha [\xi(r(x-hw) +\epsilon_1) q_h(x-hw) ]K(w)dw.
%
\end{align}
Here for $|\alpha|=1$,
\begin{align}\label{xiqderiv1}
\partial_w^\alpha [\xi(r(x+hw) +\epsilon_1) q_h(x+hw) ] = & h \xi^\prime(r(x+hw) +\epsilon_1) \partial^\alpha r(x+hw) q_h(x+hw) \nonumber\\
 & + h \xi(r(x+hw) +\epsilon_1) \partial^\alpha q_h(x+hw).
\end{align}
For $|\alpha|=2$, suppose that $\alpha=\alpha_1+\alpha_2$, where $|\alpha_1|=|\alpha_2|=1$. We have
\begin{align}\label{xiqderiv2}
&\partial_w^\alpha [\xi(r(x+hw) +\epsilon_1) q_h(x+hw) ] \nonumber \\
= & h^2 \xi^{\prime\prime}(r(x+hw) +\epsilon_1) \partial^{\alpha_1} r(x+hw) \partial^{\alpha_2} r(x+hw) q_h(x+hw) \nonumber \\
& + h^2 \xi^\prime(r(x+hw) +\epsilon_1) \partial^\alpha r(x+hw) q_h(x+hw) \nonumber\\
& + h^2 \xi^\prime(r(x+hw) +\epsilon_1) \partial^{\alpha_1} r(x+hw) \partial^{\alpha_2} q_h(x+hw) \nonumber\\
& + h^2 \xi^\prime(r(x+hw) +\epsilon_1) \partial^{\alpha_2} r(x+hw) \partial^{\alpha_1} q_h(x+hw) \nonumber\\
& + h^2 \xi(r(x+hw) +\epsilon_1) \partial^\alpha q_h(x+hw).
\end{align}

Using a Taylor expansion of order 2, we have 
\begin{align}\label{xitaylor0}
\sup_{x\in\mathcal{X}}|\xi(r(x+hw) +\epsilon_1) - \xi(r(x) +\epsilon_1) - h\xi^\prime(r(x) +\epsilon_1) w^T \nabla r(x)| \leq  h^2 d c_{\xi}c_r\|w\|^2.
\end{align}
and
\begin{align}\label{xitaylor1}
\sup_{x\in\mathcal{X}}|\xi^\prime(r(x+hw) +\epsilon_1) - \xi^\prime(r(x) +\epsilon_1) - h\xi^{\prime\prime}(r(x) +\epsilon_1) w^T \nabla r(x)| \leq  h^2 d c_{\xi}c_r\|w\|^2.
\end{align}
Using a Taylor expansion of order 1, we have 
\begin{align}\label{xitaylor2}
\sup_{x\in\mathcal{X}}|\xi^{\prime\prime}(r(x+hw) +\epsilon_1) - \xi^{\prime\prime}(r(x) +\epsilon_1) | \leq  h \sqrt{d} c_{\xi}c_r\|w\|.
\end{align}
For $|\alpha|=1$, using a Taylor expansion of order 2 we have 
\begin{align}\label{rtaylor1}
\sup_{x\in\mathcal{X}} |\partial^\alpha r(x+hw) - \partial^\alpha r(x) + hw^T \nabla\partial^\alpha r(x)|\leq h^2 d c_r\|w\|^2.
\end{align}
For $|\alpha|=2$, using a Taylor expansion of order 1 we have 
\begin{align}\label{rtaylor2}
\sup_{x\in\mathcal{X}} |\partial^\alpha r(x+hw) - \partial^\alpha r(x) |\leq h \sqrt{d} c_r\|w\|.
\end{align}
Therefore it follows from (\ref{expalphar}), (\ref{qhrate0}), and (\ref{xitaylor0}) that
\begin{align}\label{r0diff0}
\sup_{x\in\mathcal{X}} |\mathbb{E} \wh r_0(x) - \wt r_0(x)| \leq (C_{q,0}c_{\xi} + dc_{\xi} c_rc_K)h^2 := C_{r,0}h^2.
\end{align}
For $|\alpha|=1$, the calculations in (\ref{expalphar}), (\ref{qhrate0}), (\ref{qhrate1}), (\ref{xiqderiv1}), (\ref{xitaylor0}), (\ref{xitaylor1}) and (\ref{rtaylor1}) yield
\begin{align}\label{r0diff1}
\sup_{x\in\mathcal{X}} |\mathbb{E} \partial^\alpha \wh r_0(x) - \partial^\alpha \wt r_0(x)| \leq (C_{q,1}c_{\xi}+C_{q,0}c_{\xi} c_r + 3d c_r^2c_Kc_{\xi})h^2:= C_{r,1}h^2.
\end{align}
For $|\alpha|=2$, using \ref{expalphar}), (\ref{qhrate0}) - (\ref{qhrate2}), and (\ref{xiqderiv2}) - (\ref{rtaylor2}) we get
\begin{align}\label{r0diff2}
&\sup_{x\in\mathcal{X}} |\mathbb{E} \partial^\alpha \wh r_0(x) - \partial^\alpha \wt r_0(x)| \nonumber\\
\leq & (C_{q,2}c_{\xi} + C_{q,1}c_{\xi}c_r + 3C_{q,0}c_{\xi}c_r^2 + 4dc_{\xi}c_r^2c_K)h^2  + 2\sqrt{d} c_{\xi}c_r^3c_Kh \nonumber\\
\leq & [\delta(C_{q,2}c_{\xi} + C_{q,1}c_{\xi}c_r + 3C_{q,0}c_{\xi}c_r^2 + 4dc_{\xi}c_r^2c_K) + 2\sqrt{d} c_{\xi}c_r^3c_K] h:= C_{r,2}h.
\end{align}
The proof of this proposition is completed with a constant $C=\max\{C_{r,0}, C_{r,1}, C_{r,2}\}$.
%
\end{proof}
%

\subsection{Proofs of Lemma~\ref{modelemma} and Theorem~\ref{modebound1}}
Theorem~\ref{modebound1} is a direct consequence of the application of Lemma~\ref{modelemma} and Theorem~\ref{derbound}, so we only give the proof of Lemma~\ref{modelemma} below.


\begin{proof}
Let $c_p=\sup_{x\in\mathcal{R}}\sup_{|\beta|\leq 3}|\partial^\beta p(x)|$. Let $\lambda_\dagger= \inf_{x\in\mathcal{C}}|\lambda_1(x)|$. Since $p$ is a Morse function and $\mathcal{R}$ is a compact set, we have $\lambda_*\geq \lambda_\dagger>0$. Let $\kappa= \frac{\lambda_\dagger}{2dc_p}\wedge\eta$. Then $\mathcal{M}^{\kappa}\subset \mathcal{R}.$  Let $\mathcal{C}_\circ^\kappa=\{y\in\mathcal{R}: \inf_{x\in\mathcal{C}}\|x-y\|<\kappa\}$ be the interior of $\mathcal{C}^\kappa$, and $\mathcal{T}=\mathcal{R}\backslash\mathcal{C}_\circ^\kappa$. Let $\theta = \inf_{x\in\mathcal{T}}\max_{|\alpha|=1}|\partial^\alpha p(x)|$. Note that $\theta>0$ when $\mathcal{T}\neq\emptyset$, because $\max_{|\alpha|=1}|\partial^\alpha p|$ is a continuous function on $\mathcal{R}$ and $\mathcal{T}$ is a compact set. We will show the result in this lemma holds when the following three conditions are satisfied.
\begin{align}
&\wt\delta_0:=\sup_{x\in\mathcal{R}}|p(x) - \wt p(x)| < \frac{1}{8}\lambda_*\kappa^2, \label{condition0}\\
&\wt\delta_2:=\sup_{x\in \mathcal{R}}\max_{|\alpha|=2} |\partial^\alpha p(x) - \partial^\alpha \wt p(x)|\leq \frac{\lambda_\dagger}{4d}, \label{condition2}\\
&\wt\delta_1:= \sup_{x\in \mathcal{R}}\max_{|\alpha|=1} |\partial^\alpha p(x) - \partial^\alpha \wt p(x)|\leq \frac{1}{2} \theta, \; \text{ when } \mathcal{T}\neq\emptyset. \label{condition1}
\end{align}
%
{\bf Step 1.} For any $x\in\mathcal{M}$, consider any $y\in \mathcal{B}_x^\kappa:=\{y\in\mathcal{R}: \|x-y\|\leq \kappa\}$, and using Weyl's inequality (see page 15 of Serre, 2002) we have 
\begin{align*}
|\lambda_1(y) - \lambda_1(x)|  \leq d \sup_{|\beta| = 2}|\partial^\beta p(x) - \partial^\beta p(y)| \leq d \sup_{z\in\mathcal{R}}\sup_{|\beta| = 3}|\partial^\beta p(z)| \|x-y\| \leq dc_p\|x-y\|.
\end{align*}
Therefore for all $y\in \mathcal{B}_x^\kappa$,
\begin{align}\label{lambda1diff}
\lambda_1(y) \leq - \lambda_* + dc_p\kappa\leq -\frac{1}{2}\lambda_*.
\end{align}
In other words, $\mathcal{M}^{\kappa}\subset \mathcal{A}: =\{x\in\mathcal{R}: \lambda_1(x) \leq -\frac{1}{2}\lambda_*\}$. For all $x\in\mathcal{M}$ and all $y\in \mathcal{B}_x^\kappa$, using a Taylor expansion we have
\begin{align}
p(y) \leq p(x) + \frac{1}{2}\sup_{z\in \mathcal{M}^{\kappa}}\lambda_1(z)\|x-y\|^2\leq p(x) - \frac{1}{4}\lambda_*\|x-y\|^2.
\end{align}
Then by using (\ref{condition0}) we must have for all $x\in\mathcal{M}$ and all $y\in\mathcal{R}$ such that $\|x-y\|=\kappa$,
\begin{align}
\wt p(y) < p(y) + \frac{1}{8}\lambda_*\kappa^2 \leq p(x) - \frac{1}{8}\lambda_*\kappa^2 < \wt p(x).
\end{align}
Therefore there must exist at least one local mode of $\wt p$ on $\mathcal{B}_x^\kappa$ for each $x\in\mathcal{M}$.

{\bf Step 2a.} Let $\mathcal{S}=\mathcal{C}\backslash\mathcal{M}$ the set of critical points of $p$ on $\mathcal{R}$ excluding local modes. Suppose $\mathcal{S} \neq \emptyset$. Then following a similar calculation in (\ref{lambda1diff}), we have that for all $y\in \mathcal{S}^{\kappa}$, $\lambda_1(y) \geq \frac{1}{2}\lambda^\dagger>0.$ For any $x\in\mathcal{R}$, let $\wt\lambda_1(x)$ be the largest eigenvalue of $\nabla^2 \wt p(x)$. For all $y\in\mathcal{S}^\kappa$, by using (\ref{condition2}) and Weyl's inequality we have
\begin{align}
\wt\lambda_1(y) \geq \lambda_1(y) - d \sup_{|\beta| = 2}|\partial^\beta \wt p(y) - \partial^\beta p(y)| \geq  \frac{\lambda_\dagger}{4}.
\end{align}
So there are no local modes of $\wt p$ on $\mathcal{S}^\kappa$. The same statement is trivially true when $\mathcal{S} = \emptyset$ because $\mathcal{S}^\kappa= \emptyset$ in such a case. 

{\bf Step 2b.} If $\mathcal{T}=\emptyset,$ then we must have $\wt{\mathcal{M}}\subset \mathcal{M}^\kappa$ based on the arguments in {\bf Step 1} and {\bf Step 2a}, since $\mathcal{R}=\mathcal{M}^\kappa\cup\mathcal{S}^\kappa\cup\mathcal{T}.$ Otherwise for all $y\in \mathcal{T}$, by using (\ref{condition1}) we have that 
%
%
\begin{align}
\max_{|\alpha|=1} |\partial^\alpha \wt p(y)| \geq \theta - \wt\delta_1 \geq \frac{1}{2} \theta>0.
\end{align}
This means that there are no local modes of $\wt p$ on $\mathcal{T}$, and hence $\wt{\mathcal{M}}\subset \mathcal{M}^\kappa.$


{\bf Step 2c.} 
Suppose that there exists $x\in\mathcal{M}$ such that there are at least two different local modes $\wt x_1$ and $\wt x_2$ of $\wt p$ within $\mathcal{B}_x^\kappa$. For any $y\in\mathcal{B}_x^\kappa$, by using (\ref{lambda1diff}), (\ref{condition2}) and Weyl's inequality, we have that 
\begin{align}\label{wtlambda1bound}
\wt\lambda_1(y) \leq \lambda_1(y) + d \sup_{|\beta| = 2}|\partial^\beta \wt p(y) - \partial^\beta p(y)| \leq  -\frac{\lambda_*}{4}.
\end{align}
Using a Taylor expansion we have 
\begin{align}
0 = (\wt x_1 - \wt x_2)^T[\nabla \wt p(\wt x_1) - \nabla \wt p(\wt x_2)] \leq \sup_{y\in\mathcal{B}_x^\kappa}\wt\lambda_1(y) \|\wt x_1 - \wt x_2\|^2, 
\end{align}
which leads to a contradiction with (\ref{wtlambda1bound}). Hence there exists only one local mode $\wt x$ of $\wt p$ in $\mathcal{B}_x^\kappa$ for each $x\in\mathcal{M}.$ For the same reason, using (\ref{lambda1diff}) it can be seen that there exists only one local mode $x$ of $p$ in $\mathcal{B}_x^\kappa$. In other words, we have that the number of modes of $p$ and $\wt p$ are the same and can be matched in such a way that 
\begin{align} \label{Hausd}
d_H(\mathcal{M} ,\wt{\mathcal{M}})=\max_{x\in\mathcal{M}}\|\wt x - x\|.
\end{align}

{\bf Step 3.} Let us consider any local mode of $p$, denoted by $x$ and its corresponding local mode $\wt x$ of $\wt p$ in $\mathcal{B}_x^\kappa$. Let $|\cdot|_{\max}$ and $\|\cdot\|_{\text{op}}$be the element-wise maximum and the operator norm of a matrix, respectively. Since $\nabla p(x)=\nabla \wt p(\wt x)=0,$ using a Taylor expansion, we have
\begin{align}
\nabla p(x) - \nabla \wt p(x) = \nabla \wt p(\wt x) - \nabla \wt p(x) = [\nabla^2 p(x) + \Delta(\wt x, x)](\wt x - x),
\end{align}
where $\Delta(\wt x, x)$ is a $d\times d$ symmetric matrix such that $|\Delta(\wt x, x)|_{\max} \leq  \wt\delta_2 + c_p |\wt x - x|_{\max}$. Therefore
\begin{align*}
\|\nabla p(x) - \nabla \wt p(x)\| & \geq \|\nabla^2 p(x)(\wt x - x)\| - \|\Delta(\wt x, x)(\wt x - x)\|\\
& \geq \lambda_*\|\wt x - x\| - \|\Delta(\wt x, x)\|_{\text{op}}\|\wt x - x\|\\
& \geq \lambda_*\|\wt x - x\| - d[\wt\delta_2 + c_p |\wt x - x|_{\max}]\|\wt x - x\|\\
& \geq \frac{1}{2}\lambda_*\|\wt x - x\| - d\wt\delta_2 \|\wt x - x\|\\
& \geq \frac{1}{4}\lambda_*\|\wt x - x\|,
\end{align*}
where in the last step we use (\ref{condition2}). The conclusion of this lemma follows by noticing (\ref{Hausd}).
\end{proof}

\subsection{Proof of Theorem 3.3}
\begin{proof}
First of all, similar to Theorem 3.1, there exist constants $C_1>0$, $c_1>0$ and $h_1>0$ such that for all $|\alpha|\leq 2$, $n\geq 1$, $0<h\leq h_1$, $\tau>1$ satisfying $nh^d\geq c_1(\tau \vee |\log h|)$ we have with probability at least $1 - 3e^{-\tau}$,
\begin{align}\label{rxbound}
\sup_{x\in\mathcal{X}}|\partial^\alpha\wh r(x) - \partial^\alpha r(x)| < C_1(\sqrt{\tau \vee |\log h|} \gamma_{n,h}^{(|\alpha|)} + h^{(3-|\alpha|)\wedge 2}).
\end{align}

Recall $\wh t$ and $\wb r$ that have been defined in (3.7) and (2.8), respectively. Note that $\wh t(x)$ corresponds to $\wh r(x)$ in the case of $\wt Y_i=1$, for all $i=1,\cdots,n.$ Let $\wh t_0(x) = \frac{1}{nh^d} \sum_{i=1}^n \frac{K_h(x-X_i)}{f_h(X_i)}$. Then similar to Proposition 5.2, there exist constants $C_2>0$, $c_2>0,$ and $h_2>0$ such that for all $|\alpha|\leq 2$, $n\geq 1$, $0<h\leq h_2$, $\tau>1$ satisfying $nh^d\geq c_2(\tau \vee |\log h|)$ we have with probability at least $1 - 2e^{-\tau}$,
\begin{align}\label{tncomp1}
%
&\sup_{x\in\mathcal{X}}|\partial^\alpha\wh t(x) - \partial^\alpha \wh t_0(x)| < C_2\sqrt{\tau \vee |\log h|} \gamma_{n,h}^{(|\alpha|)}.  
%
%
\end{align}
Similar to Proposition 5.3, there exist constants $C_3>0$, $c_3>0,$ and $h_3>0$ such that for all $|\alpha|\leq 2$, $n\geq 1$, $0<h\leq h_3$, $\tau>1$ satisfying $nh^d\geq c_3(\tau \vee |\log h|)$ we have with probability at least $1 - e^{-\tau}$,
\begin{align}\label{tncomp2}
&\sup_{x\in\mathcal{X}}|\partial^\alpha \wh t_0(x) - \mathbb{E} \partial^\alpha \wh t_0(x)| < C_3\sqrt{\tau \vee |\log h|} \gamma_{n,h}^{(|\alpha|)}. 
\end{align}
Recall that $q_h=f/f_h.$ Note that 
\begin{align*}
\mathbb{E} \partial^\alpha \wh t_0(x) = \int_{\mathbb{R}^d}  \partial^\alpha q_h(x-hw) K(w)dw 
\end{align*}
Similar to Proposition 5.4, there exist constants $C_4>0$ and $h_4>0$ such that for all $0<h\leq h_4$, and $|\alpha|\leq 2$ we have 
\begin{align}\label{tnbias}
\sup_{x\in\mathcal{X}} |\mathbb{E} \partial^\alpha\wh t_0(x) - b_{|\alpha|} | \leq C_4h^{(3-|\alpha|)\wedge 2}
\end{align}
where $b_{|\alpha|}=0$ when $|\alpha|=0$ and $b_{|\alpha|}=1$ when $|\alpha|=1,2$. Hence combing (\ref{tncomp1}), (\ref{tncomp2}) and (\ref{tnbias}), for all $|\alpha|\leq 2$ we get with probability at least $1 - 3e^{-\tau}$ that
\begin{align}
%
&\sup_{x\in\mathcal{X}}|\partial^\alpha\wh t(x) -b_{|\alpha|} | < C_5(\sqrt{\tau \vee |\log h|} \gamma_{n,h}^{(|\alpha|)} +h^{(3-|\alpha|)\wedge 2}),\label{wht12}
\end{align}
where $C_5=(C_2+C_3)\vee C_4.$ Due to the almost sure boundedness of $\pi(Y_{[n]})$, using (\ref{rxbound}), (\ref{wht12}), and the relations $\wh r_*(x) = \wh r(x) + \pi(Y_{[n]}) \wh t(x)$ and $\partial^\alpha \wb r = \partial^\alpha r + b_{|\alpha|}\pi(Y_{[n]})$, for all $|\alpha|\leq 2$ we have with probability at least $1 - 6e^{-\tau}$ that
\begin{align*}
\sup_{x\in\mathcal{X}}|\partial^\alpha\wh r_*(x) - \partial^\alpha \wb r(x)| < C_6(\sqrt{\tau \vee |\log h|} \gamma_{n,h}^{(|\alpha|)} + h^{(3-|\alpha|)\wedge 2}),
\end{align*}
where $C_6=C_1+C_5(B\vee1)$, where $B$ is given in assumption {\bf E$^\prime$}. Then the conclusion of the theorem follows from the application of Lemma~\ref{modelemma}.
\end{proof}

\section{Appendix}
%


\subsection{Proof of Proposition~\ref{talagrand}}
\begin{proof}
For any measurable function $g$ and probability measure $Q$ on $\mathbb{R}^{d+1}$, let $\|g\|_{L_2(Q)}=[\int_{\mathbb{R}^{d+1}}|g(u)|^2dQ(u)]^{1/2}$ be the $L_2(Q)$-norm of $g$. We first show that $\wt{\mathcal{G}}_h:=\{h^d(g-\mathbb{P}g): g\in\mathcal{G}_h\}$ is a uniformly bounded VC-class, where $\mathbb{P}g = \mathbb{E}g(X,Y)$. For any $x\in\mathbb{R}^d$, let $m_{x,h}(u,v)=M((x-u)/h)$ for all $(u,v)\in \mathbb{R}^d\times\mathbb{R}$. Let $B\in(0,\infty)$ be a constant envelope of $\mathscr{F}$ such that $\sup_{g\in\mathscr{F}}\sup_{x\in\mathbb{R}^{d+1}}|g(x)|\leq B$. Define $\mathscr{F}_h=\{ m_{x,h}(\cdot): x\in \mathbb{R}^d\}$, and $\mathscr{F}_{h,\mathcal{X}}=\{ m_{x,h}(\cdot): x\in \mathcal{X}\}$ for all $h>0$. Using Lemma A.3 of of Sriperumbudur and Steinwart (2012), we obtain that for all $h>0$, and $\epsilon\in(0,B]$,
\begin{align}\label{fhxcovering}
\sup_Q \mathcal{N}(\mathscr{F}_{h,\mathcal{X}}, L_2(Q),\epsilon) \leq \sup_Q \mathcal{N}(\mathscr{F}_h, L_2(Q),\epsilon) = \sup_Q \mathcal{N}(\mathscr{F}, L_2(Q),\epsilon) \leq \Big( \frac{AB}{\epsilon}\Big)^\nu,
\end{align}
where the supremum is taken over all the probability measures $Q$ on $\mathbb{R}^{d+1}.$

Let $\mathscr{F}_{h,\mathcal{X}}^{(1)}=\{\mathcal{X}\times\mathbb{R}\ni(u,v)\mapsto\zeta_h(u,v)m_{x,h}(u,v):\;x\in\mathcal{X}\}$. Note that $B^{(1)}:=LB$ is a constant envelope of $\mathscr{F}_{h,\mathcal{X}}^{(1)}$ such that $\sup_{g\in\mathscr{F}_h^{(1)}}\sup_{x\in\mathbb{R}^{d+1}}|g(x)|\leq B^{(1)}$. It follows from (\ref{fhxcovering}) that, for any given probability measure $Q$ on $\mathbb{R}^{d+1}$ and any $\epsilon\in(0,B^{(1)}]$, there exist $x_1,\cdots,x_{N_\epsilon}\in\mathcal{X}$ with $N_\epsilon\leq ( \frac{AB^{(1)}}{\epsilon})^\nu$ such that $\{m_{x_{j},h}:\; j=1,\cdots,N_\epsilon\}$ is an $(\frac{\epsilon}{L})$-covering of $\mathscr{F}_{h,\mathcal{X}}$ with respect to the $L_2(Q)$-norm. In other words, for any $x\in\mathcal{X}$, there exists $j\in\{1,\cdots,N_\epsilon\}$ such that $\|m_{x,h} - m_{x_{j},h} \|_{L(Q)} \leq \frac{\epsilon}{L}.$ 
Hence for any $m_{x,h}\in\mathscr{F}_{x,\mathcal{X}}^{(1)}$,
\begin{align*}
&\|\zeta_h m_{x,h} - \zeta_hm_{x_{j},h} \|_{L(Q)} \leq L \|m_{x,h} - m_{x_{j},h} \|_{L(Q)}\leq \epsilon.
%
\end{align*}
This means that $\{\zeta_hm_{x_{j},h}:\; j=1,\cdots,N_\epsilon\}$ is an $\epsilon$-covering of $\mathscr{F}_{h,\mathcal{X}}^{(1)}$ with respect to the $L_2(Q)$-norm. Hence for any $m_{x,h}\in\mathscr{F}_{x,\mathcal{X}}$ and any $\epsilon\in(0,B^{(1)}]$,
\begin{align*}
\sup_Q \mathcal{N}(\mathscr{F}_{h,\mathcal{X}}^{(1)}, L_2(Q),\epsilon) \leq \Big( \frac{AB^{(1)}}{\epsilon}\Big)^\nu,
\end{align*}
which implies that, for any given probability measure $Q$ on $\mathbb{R}^{d+1}$ and any $\epsilon\in(0,2B^{(1)}]$, there exist $x_1^\prime,\cdots,x_{N_\epsilon}^\prime\in\mathcal{X}$ with $N_\epsilon\leq ( \frac{2AB^{(1)}}{\epsilon})^\nu$ such that $\{\zeta_h m_{x_{j}^\prime,h}:\; j=1,\cdots,N_\epsilon\}$ is a $(\frac{1}{2}\epsilon)$-covering of $\mathscr{F}_{h,\mathcal{X}}^{(1)}$ with respect to the $L_2(Q)$-norm.

Consider the interval $[-B^{(1)}, B^{(1)}]$. For any $\epsilon>0$, there exist $b_1,\cdots,b_{N_{\epsilon}}$ with $N_\epsilon\leq \lceil2B^{(1)}/\epsilon\rceil$ such that $b_1,\cdots,b_{N_{\epsilon}}$ is a $(\frac{1}{2}\epsilon)$-covering of $[-B^{(1)}, B^{(1)}]$, where $\lceil\cdot \rceil$ is the ceiling function. Let $\mathscr{F}_h^{(2)}=\{g(\cdot)-b:\; g\in \mathscr{F}_{h,\mathcal{X}}^{(1)}, |b|\leq B^{(1)}\}$. For any $g\in \mathscr{F}_{h,\mathcal{X}}^{(1)}$ and $|b|\leq B^{(1)}$, there exist $m_{x_{i}^\prime,h}$ and $b_j$ such that $\|m_{x,h} - m_{x_{i}^\prime,h} \|_{L(Q)} \leq \frac{1}{2}\epsilon$ and $|b-b_j|\leq \frac{1}{2}\epsilon$. Hence
\begin{align}
\|(m_{x,h} -b) - (m_{x_{i}^\prime,h} -b_j) \|_{L_2(Q)} \leq \|m_{x,h} - m_{x_{i}^\prime,h} \|_{L_2(Q)} + |b-b_j|  \leq \epsilon.
\end{align}
Therefore with $A^{(2)} = 2(A\vee 1)$ and $B^{(2)} = 2B^{(1)}$ we have
\begin{align*}
\sup_Q \mathcal{N}(\mathscr{F}_{h,\mathcal{X}}^{(2)}, L_2(Q),\epsilon) \leq \Big( \frac{2AB^{(1)}}{\epsilon}\Big)^\nu \lceil2B^{(1)}/\epsilon\rceil \leq \Big( \frac{A^{(2)}B^{(2)}}{\epsilon}\Big)^{\nu+1}.
\end{align*}
Note that $\sup_{g\in\mathscr{F}_h^{(2)}}\sup_{x\in\mathbb{R}^{d+1}}|g(x)|\leq B^{(2)}$. Since $\wt{\mathcal{G}}_h \subset \mathscr{F}_{h,\mathcal{X}}^{(2)}$, we have
\begin{align}
\sup_Q \mathcal{N}(\wt{\mathcal{G}}_h, L_2(Q),\epsilon) \leq \sup_Q \mathcal{N}(\mathscr{F}_{h,\mathcal{X}}^{(2)}, L_2(Q),\epsilon) \leq \Big( \frac{A^{(2)}B^{(2)}}{\epsilon}\Big)^{\nu+1}.
\end{align}
This then shows that $\wt{\mathcal{G}}_h$ is a VC class with characteristics $A^{(2)}$ and $\nu+1$, and is uniformly bounded by a constant envelope $B^{(2)}$. For any $g\in \wt{\mathcal{G}}_h$, for all $h\in(0,h_0]$, we have
\begin{align*}
\mathbb{P} g^2 = \mathbb{E}g^2(X,Y) &\leq \mathbb{E}\Big\{\Big[\zeta_h(X,Y)M\Big(\frac{x-X}{h}\Big)\Big]^2\Big\} \\
&\leq L^2 \mathbb{E}\Big\{\Big[M\Big(\frac{x-X}{h}\Big)\Big]^2\Big\} \\
& \leq h^d L^2\int_{\mathbb{R}^d} [M(w)]^2 f(x-hw) dw \\
&\leq h^d L^2 \sup_{x\in\mathcal{X}^{\eta h_0}}|f(x)| \|M\|_2^2 \\
&:= h^d \sigma_0^2.
\end{align*}
Applying Theorems A.1 and A.2 of Sriperumbudur and Steinwart (2012) we have that for all $h\in(0,h_0]$, $n\geq 1$ and $\tau>0$, with probability at least $1-e^{-\tau}$,
\begin{align*}
&\frac{1}{\sqrt{n}}\sup_{g\in \mathcal{G}_h } |\mathbb{G}_n(g)| \\
\leq & 4 \frac{1}{\sqrt{n}} \mathbb{E}\sup_{g\in \mathcal{G}_h }|\mathbb{G}_n(g)| + \sqrt{\frac{2\tau \sigma_0^2}{nh^d }} + \frac{\tau B^{(2)}}{nh^d} \\
\leq & 4C\Big[ \frac{(\nu+1) B^{(2)}}{nh^d}\log\frac{A^{(2)}B^{(2)}}{\sqrt{h^d \sigma_0^2}} + \sqrt{\frac{(\nu+1)\sigma_0^2}{nh^d}} \log\frac{A^{(2)}B^{(2)}}{\sqrt{h^d \sigma_0^2}} \Big]  + \sqrt{\frac{2\tau \sigma_0^2}{nh^d }} + \frac{\tau B^{(2)}}{nh^d},
\end{align*}
where $C$ is a universal constant that is given in Theorem A.2 of Sriperumbudur and Steinwart (2012).
\end{proof}

\section*{Acknowledgements}
This work is partially supported by grants NSF DMS 1821154 and NSF FET 1900061. This material is additionally based upon work by AS supported by (while serving at) the National Science Foundation.  Any opinion, findings, and conclusions or recommendations expressed in this material are those of the author(s) and do not necessarily reflect the views of the National Science Foundation.


\begin{thebibliography}{}
\setlength{\itemindent}{-\leftmargin}
\makeatletter\renewcommand{\@biblabel}[1]{}\makeatother
\bibitem{} E. Arias-Castro, D. Mason, and B. Pelletier (2016). 
\newblock On the estimation of the gradient lines of a density and the consistency of the mean-shift algorithm.
\newblock \textit{Journal of Machine Learning Research}, \textbf{17}: 1-28.

\bibitem{} Bertsekas, D.P. (1999)
\newblock \textit{Nonlinear Programming}, 2nd edition, Athena Scientific, Belmont, Massachusetts.

\bibitem{} L. Breiman, J.H. Friedman, R.A. Olshen, C.J. Stone (1993). 
\newblock \textit{Classification and Regression Trees}. Chapman and Hall.


\bibitem{} J.E. Ch\'{a}con (2015).
\newblock A population background for nonparametric density-based clustering.
\newblock \textit{Statistical Science}, \textbf{30}(4): 518--532.


\bibitem{} Y.-C. Chen, C. R. Genovese, J. Tibshirani, and L. Wasserman (2016).
\newblock Nonparametric modal regression.
\newblock \textit{Ann. Statist.} \textbf{44}(2): 489-514.

\bibitem{} Y.-C. Chen, C. R. Genovese, and L. Wasserman (2016). 
\newblock A comprehensive approach to mode clustering.
\newblock \textit{Electron. J. Statist.} \textbf{10}(1): 210--241.

\bibitem{} Y. Cheng (1995).
    \newblock Mean shift, mode seeking, and clustering.
    \newblock \textit{IEEE Transactions on Pattern Analysis and Machine Intelligence}, \textbf{17}(8):790--799.

\bibitem{} V. Chernozhukov, D. Chetverikov, and K. Kato (2013).
    \newblock Gaussian approximation of suprema of empirical processes.
\newblock \textit{Ann. Statist.} \textbf{42}(4): 1564-1597.  

\bibitem{} R. Clausen, B. Ma, R. Nussinov, and A. Shehu (2015).
\newblock Mapping the Conformation Space of Wildtype and Mutant H-Ras with a Memetic, Cellular, and Multiscale Evolutionary Algorithm.
\newblock \textit{PLoS Computational Biology} \textbf{11}(9).

\bibitem{} D. Comaniciu and P. Meer (2002).
    \newblock Mean shift: A robust approach toward feature space analysis.
    \newblock \textit{IEEE Transactions on Pattern Analysis and Machine Intelligence} \textbf{24}(5):1--18.

\bibitem{} D. Comaniciu, V. Ramesh, and P. Meer (2002).
\newblock Kernel-based object tracking.
\newblock \textit{IEEE Trans. Pattern Anal. Mach. Intell.} \textbf{25}(5): 564--577.


\bibitem{} J. Einbeck and G. Tutz (2006).
\newblock Modelling beyond regression functions: An application of multimodal regression to speed-flow data.
\newblock \textit{J. Roy. Statistical Soc.: Series C (Appl. Statist.)} \textbf{55}(4): 461--475.

\bibitem{} U. Einmahl and D.M. Mason (2000).
\newblock Uniform in bandwidth consistency of kernel-type function estimators.
\newblock \textit{Annals of Statistics} \textbf{33}: 1380--1403.

\bibitem{} K. Fukunaga and L. D. Hostetler (1975).
    \newblock The estimation of the gradient of a density function, with applications in pattern recognition.
     \newblock \textit{EEE Transactions on Information Theory} \textbf{21}(1):32--40.

\bibitem{}  J. Gaudart, N. Graffeo, G. Barbet, S. Rebaudet, N. Dessay, O. Doumbo, and R. Giorgi (2015). 
\newblock SPODT: An R Package to Perform Spatial Partitioning. 
\newblock \textit{Journal of Statistical Software}, \textbf{63}(16).

\bibitem{}  J. Gaudart, B. Poudiougou, S. Ranque, and O. Doumbo (2005). 
\newblock Oblique decision trees for spatial pattern detection: optimal algorithm and application to malaria risk. 
\newblock \textit{BMC Medical Research Methodology}, \textbf{5}(1), 1--11.

\bibitem{}  C.R. Genovese, M. Perone-Pacifico, I. Verdinelli, and L. Wasserman (2014). 
\newblock Nonparametric ridge estimation. 
\newblock \textit{Annals of Statistics}, \textbf{42}(4), 1511--1545.

\bibitem{} Y. A. Ghassabeh (2015).
\newblock A sufficient condition for the convergence of the mean shift algorithm with Gaussian kernel.
\newblock \textit{Journal of Multivariate Analysis} \textbf{135}: 1--10. 

\bibitem{} E. Gin\'{e} and A. Guillou (2002).
\newblock Rates of strong uniform consistency for multivariate kernel density estimators.
\newblock \textit{Annals of the Institute Henri Poincar\'{e}: Probability and Statistics}, \textbf{38}: 907--921.

\bibitem{} D.J. Henderson, Q. Li, C.F. Parmeter, and S. Yao (2015).
\newblock Gradient-based smoothing parameter selection for nonparametric regression estimation.
\newblock \textit{Journal of Econometrics} \textbf{184}: 233--241.

\bibitem{} D.J. Henderson, and C.F. Parmeter (2015).
\newblock \textit{Applied Nonparametric Econometrics},
\newblock Cambridge University Press.

\bibitem{} H. Jiang (2019).
\newblock Non-asymptotic uniform rates of consistency for k-NN regression.
\newblock \textit{Proceedings of the AAAI Conference on Artificial Intelligence} \textbf{33}(1): 3999-4006.

\bibitem{} Y. P. Mack, and H.-G. M\"{u}ller (1989).
    \newblock Derivative estimation in non-parametric regression with random predictor variables.
    \newblock \textit{Sankhya} \textbf{51}:59--72, Ser. A.

\bibitem{}  T. Maximova, E. Plaku, and A. Shehu (2016).
\newblock Structure-guided protein transition modeling with a probabilistic roadmap algorithm. 
\newblock \textit{IEEE/ACM transactions on computational biology and bioinformatics}, \textbf{15}(6), 1783--1796.

\bibitem{}  T. Maximova, Z. Zhang, D. B. Carr, E. Plaku, and A. Shehu (2018). 
\newblock Sample-based models of protein energy landscapes and slow structural rearrangements. 
\newblock \textit{Journal of Computational Biology}, \textbf{25}(1): 33--50.

\bibitem{} J. Legewie (2018). 
\newblock Living on the edge: neighborhood boundaries and the spatial dynamics of violent crime. 
\newblock \textit{Demography}, \textbf{55}(5), 1957--1977.

\bibitem{} B. Liu, B. Mavrin, D. Niu, and L. Kong (2016). 
\newblock House price modeling over heterogeneous regions with hierarchical spatial functional analysis. 
\newblock In \textit{2016 IEEE 16th International Conference on Data Mining (ICDM)}, pp. 1047-1052.    

\bibitem{} J. Milnor (1963).
\newblock  \textit{Morse Theory}, Princeton University Press.

\bibitem{} H.-G. M\"{u}ller (1985).
\newblock Kernel estimators of zeros and of location and size of extrema of regression functions.
\newblock \textit{Scandinavian Journal of Statistics} \textbf{12}(3): 221--232.

\bibitem{} H.-G. M\"{u}ller (1989).
\newblock Adaptive nonparametric peak estimation.
\newblock \textit{Annals of Statistics} \textbf{17}(3): 1053--1069.

\bibitem{} D. Nolan and D. Pollard (1987).
\newblock $U$-processes: rates of convergence.
\newblock \textit{Annals of Statistics} \textbf{15}(2): 780--799.

\bibitem{} U. Ozertem, and D. Erdogmus, (2011). 
\newblock Locally defined principal curves and surfaces. 
\newblock \textit{The Journal of Machine Learning Research}, \textbf{12}, 1249--1286.

\bibitem{} W. Qiao, and W. Polonik (2016). 
\newblock Theoretical analysis of nonparametric filament estimation. 
\newblock \textit{Annals of Statistics}, \textbf{44}(3), 1269--1297.

\bibitem{} D. Serre (2002).
\newblock {\em Matrices: Theory and Applications}. Springer-Verlag, New York.
    
\bibitem{} B. Sriperumbudur, I. Steinwart (2012).
\newblock Consistency and rates for clustering with DBSCAN.
\newblock \textit{Proceedings of the Fifteenth International Conference on Artificial Intelligence and Statistics, PMLR} \textbf{22}: 1090--1098.

\bibitem{} A.B. Tsybakov (1990).
\newblock Recursive estimation of the mode of a multivariate distribution. 
\newblock \textit{Problemy Peredachi Informatsii}, \textbf{26}(1), 38--45.

\bibitem{} R. Yamasaki and T. Tanaka (2020).
\newblock Properties of mean shift.
\newblock \textit{IEEE Transactions on Pattern Analysis and Machine Intelligence} \textbf{42}(9): 2273 -- 2286.

\bibitem{} K. Ziegler (2002).
\newblock On nonparametric kernel estimation of the mode of the regression function in the random design model.
\newblock \textit{Journal of Nonparametric Statistics} \textbf{14}(6): 749--774.

\end{thebibliography}
\end{document}